\documentclass{article} 
\usepackage[preprint]{colm2026_conference}

\usepackage{microtype}
\usepackage{hyperref}
\usepackage{url}
\usepackage{booktabs}

\usepackage{url}            
\usepackage{hyperref} 
\usepackage{xcolor}
\usepackage[ruled,vlined]{algorithm2e}
\definecolor{curveblue}{HTML}{00004D} 
\usepackage{graphicx} 
\usepackage{multirow}
\usepackage{booktabs}       
\usepackage{amsfonts}       
\usepackage{nicefrac}       
\usepackage{float}
\usepackage{makecell}
\usepackage{microtype}      
\usepackage{graphicx}
\usepackage{subcaption}
\usepackage{arydshln}
\usepackage{xr}
\usepackage{array}
\usepackage{longtable} 
\usepackage[title,toc,page]{appendix}
\usepackage{tabularx}

\usepackage{amsmath}
\usepackage{amssymb}
\usepackage{mathtools}
\usepackage{amsthm}
\usepackage[svgnames, x11names]{xcolor}
\usepackage{framed}
\usepackage{quoting}
\colorlet{shadecolor}{LavenderBlush2}
\colorlet{shadecolor}{Crimson}
\usepackage{microtype}
\usepackage{graphicx}
\usepackage{subcaption}
\usepackage{booktabs} 
\usepackage{tcolorbox}
\usepackage{hyperref}
\usepackage{wrapfig}

\usepackage{amsmath}
\usepackage{amssymb}
\usepackage{mathtools}
\usepackage{amsthm}

\usepackage[capitalize,noabbrev]{cleveref}

\theoremstyle{plain}
\newtheorem{theorem}{Theorem}[section]
\newtheorem{proposition}[theorem]{Proposition}
\newtheorem{lemma}[theorem]{Lemma}
\newtheorem{corollary}[theorem]{Corollary}
\theoremstyle{definition}

\theoremstyle{remark}

\usepackage[textsize=tiny]{todonotes}

\usepackage{lineno}

\definecolor{darkblue}{rgb}{0, 0, 0.5}
\hypersetup{colorlinks=true, citecolor=darkblue, linkcolor=darkblue, urlcolor=darkblue}

\title{Rethinking Token Prediction: \\Tree-Structured Diffusion Language Model}

\author{Zihao Wu, Haoming Yang, Juncheng Dong \& Vahid Tarokh \\
Department of Electrical \& Computer Engineering\\
Duke University\\
}

\begin{document}

\ifcolmsubmission
\linenumbers
\fi

\maketitle

\begin{abstract}
Discrete diffusion language models have emerged as a competitive alternative to auto-regressive language models, but training them efficiently under limited parameter and memory budgets remains challenging. Modern architectures are predominantly based on a full-vocabulary token prediction layer, which accounts for a substantial fraction of model parameters (e.g., more than $20\%$ in small scale DiT-style designs) and often dominates peak GPU memory usage. This leads to inefficient use of both parameters and memory under constrained training resources. To address this issue, we revisit the necessity of explicit full-vocabulary prediction, and instead exploit the inherent structure among tokens to build a tree-structured diffusion language model. Specifically, we model the diffusion process with intermediate latent states corresponding to a token's ancestor nodes in a pre-constructed vocabulary tree. This tree-structured factorization exponentially reduces the classification dimensionality, makes the prediction head negligible in size, and enables reallocation of parameters to deepen the attention blocks. Empirically, under the same parameter budget, our method reduces peak GPU memory usage by half while matching the perplexity performance of state-of-the-art discrete diffusion language models.
\end{abstract}

\section{Introduction}

While state-of-the-art language models are predominantly autoregressive~\citep{brown2020language, achiam2023gpt}, recent advances in discrete diffusion language models (DLMs)~\citep{lou2023discrete,sahoo2024simple,nie2025large} have highlighted a competitive and efficient alternative pathway for language modeling~\citep{lou2023discrete}. Nevertheless, existing DLMs rely on \emph{token-level} prediction, which remains a major computational bottleneck, especially at smaller model scales. Since modern tokenizers use large vocabulary, token-level prediction requires a classification layer with substantial parameter cost and memory footprint~\citep{nie2025large}. In the small- and base-scale DiT-style architectures~\citep{sahoo2024simple, rutte2025generalized}, for instance, the output projection matrix can account for more than $20\%$ and $12\%$ of the total parameters, respectively. 
\begin{figure*}[htbp]
  \centering
  \includegraphics[width=\linewidth]{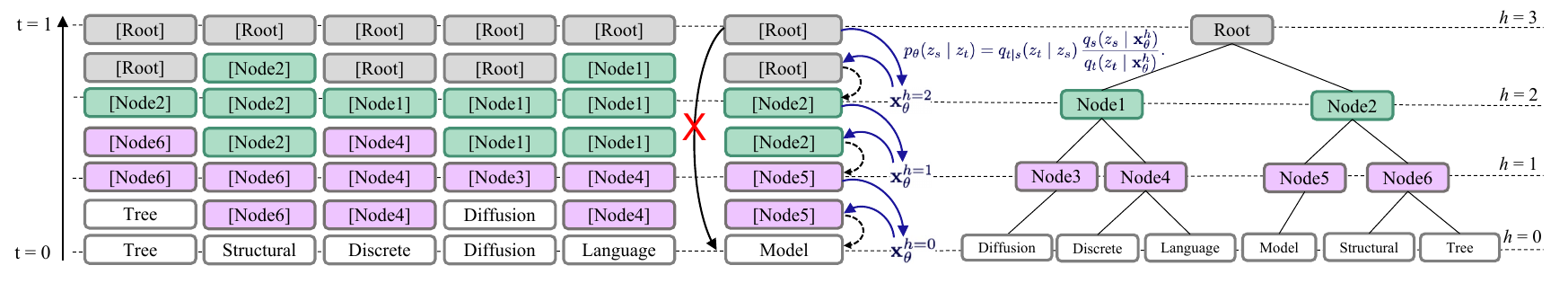}
  \caption{Generation process (left): Our algorithm features in-level child prediction rather than standard token prediction, delivering substantial efficiency gains while achieving improved performance. Token tree (right): Child prediction is enabled by a principled token tree, in which each state evolves strictly between adjacent levels.}
  \label{fig:main-graph}
\end{figure*}This output layer also dominates peak activation memory during training, directly constraining efficient training and deployment in compute-limited settings, including edge-device scenarios~\citep{shamshoum2025compact}.


To address this challenge, we observe that token-level classification in diffusion language models neglects \emph{exploitable structure} in the token space, which can be leveraged to improve model efficiency. Subword vocabulary is inherently hierarchical; semantically or syntactically related tokens often form clusters; and in many contexts, prediction only requires distinguishing among a small subset of plausible tokens \citep{mielke2021between, dar2023analyzing, kurtic2023ziplm, gao2023retrieval}. These observations suggest that token prediction often operates over a structured, context-dependent subset of the vocabulary rather than the full token set \citep{zhu2025sentencekv, palacio2023evaluating, ugare2024syncode}. Yet existing diffusion language models do not explicitly leverage this structure. This mismatch raises a natural question:

\begin{tcolorbox}[
  colback=white,
  colframe=black!18,
  boxrule=0.6pt,
  arc=2.5mm,
  left=10pt,right=10pt,top=9pt,bottom=9pt,
  before skip=10pt, after skip=10pt
]
{\itshape 
If token prediction is inherently structured and context-dependent, can a diffusion language model leverage this structure to improve model efficiency under limited training budgets?
 }
\end{tcolorbox}

One natural way to impose structure on token prediction is through the construction of a hierarchical vocabulary tree. Tokens are first organized according to a chosen criterion, such as semantic or syntactic structure \citep{radford2019language, karanikolas2023large}, and prediction is then factorized into a sequence of smaller decisions rather than a single $|V|$-way classification. This substantially reduces the effective output dimensionality and the memory cost of the classification head. Similar ideas are explored in autoregressive models prior to the transformer era through hierarchical softmax \citep{morin2005hierarchical} and adaptive softmax \citep{grave2017efficient}. However, because autoregressive prediction is coupled to left-to-right generation, hierarchical decisions are not naturally aligned with the generation process. Diffusion language models instead refine token representations iteratively, making them a more natural setting for coarse-to-fine prediction over a hierarchy.

In light of these insights, we propose the \emph{Tree-Structured Diffusion Language Model} (\textbf{TDLM}). TDLM models a pre-constructed vocabulary tree through a discrete diffusion process, where both the forward and reverse transitions are closely aligned with the parent–child relationships of the tree. This requires a modified diffusion formulation and a new derivation of the training ELBO, resulting in a training procedure that differs fundamentally from existing DLMs \citep{sahoo2024simple, rutte2025generalized}. By reducing the effective size of the classification layer, TDLM lowers training memory usage and frees parameters for the backbone, yielding substantial efficiency gains while maintaining strong modeling performance comparable to the state-of-the-art methods under limited computational resource.

\section{Preliminary}
Discrete diffusion models (DDMs) generalize diffusion from continuous spaces to a finite state space $\mathcal{Z}$. Given data $X\sim q_0$ taking values in $\mathcal{Z}$, DDMs first define a family of forward Markov transition kernels $\{q_{t|s}\}_{0\leq s\leq t\leq T}$ that progressively corrupt data $X$ into a simple noise prior at the terminal time $T=1$. During training, DDMs learn the corresponding reverse-time dynamics, generating data through iterative denoising steps.

\textbf{CTMC.} A standard approach to instantiate DDMs is via a time-inhomogeneous continuous-time Markov chain (CTMC) $\{z_t\}_{t\in [0,1]}$, characterized by a generator (i.e., forward transition rate) matrix $Q_t$ \citep{campbell2022continuous}. We represent each state $z\in \mathcal{\mathcal{Z}}$ as a one-hot encoding vector $\mathbf{z}\in \{0,1\}^{|\mathcal{Z}|}$ with $\sum_i \mathbf{z}_i = 1$. The generator then specifies the infinitesimal forward kernel by 
\begin{align*}
&q_{t \mid t-\Delta t}(z \mid \tilde z)
= \delta_{z,\tilde z} \;+\; Q_t(\tilde z, z)\,\Delta t \;+\; o(\Delta t),
\quad \Delta t \downarrow 0,
\\
&Q_t(\tilde z, z) \ge 0, \forall\, z \neq \tilde z,
\quad \text{and}\;\;\; Q_t(\tilde z, \tilde z)
= - \sum_{z \neq \tilde z} Q_t(\tilde z, z),
\end{align*}
where $\delta_{z,\tilde z}$ is the Kronecker delta function, which equals $1$ if $z=\tilde{z}$, and $0$ otherwise.

For $s<t$, let $P_{t|s}$ denote the cumulative transition matrix from time $s$ to $t$, such that 
\[
q_{t\mid s}(z_t\mid z_s)=\mathrm{Cat}\!\big(z_t;\,P_{t\mid s}\mathbf{z_s}\big),
\]
i.e., the $z_s$-th column of $P_{t\mid s}$ gives the transition probabilities starting from state $z_s$. Then, $P_{t|s}$ satisfies the Kolmogorov forward/backward equations
\[
\frac{\partial P_{t\mid s}}{\partial t}=Q_t^\top P_{t\mid s},\quad
\frac{\partial P_{t\mid s}}{\partial s}=-P_{t\mid s}Q_s^\top,\quad
P_{s\mid s}=I.
\]

\textbf{MDLM.} Masked diffusion language model (MDLM) is initially formulated as a discrete-time Markov chain, but admits a CTMC interpretation under continuous parameterization \citep{sahoo2024simple}. Let $\mathbf{m}$ be the one-hot vector for the absorbing token [MASK]. Given a clean token $x$, MDLM defines the forward marginal at time $t\in [0,1]$ by 
\begin{align*}
    q_t(z_t \mid x)
&= \mathrm{Cat}\!\bigl(z_t;\,\alpha_t\,\mathbf{x} + (1-\alpha_t)\,\mathbf{m}\bigr),
\end{align*}
where $\{\alpha_t\}_{t\in {[0,1]}}$ is a decreasing schedule with $\alpha_0 = 1$ and $\alpha_1 = 0$, so the process smoothly interpolates from the data distribution at $t=0$ to the absorbing mask at $t=1$.

\textbf{GIDD.} Generalized interpolating diffusion (GIDD) explicitly incorporates a CTMC formulation and derives the corresponding transition rates. Specifically, it extends MDLM by replacing the fixed absorbing token $\mathbf{m}$ with a time-dependent mixing distribution $\pi_t$ \citep{rutte2025generalized}, yielding the forward marginal
\begin{align*}
    q_t(z_t \mid x)
&= \mathrm{Cat}\!\bigl(z_t;\,\alpha_t\,\mathbf{x} + (1-\alpha_t)\,\boldsymbol{\pi}_t\bigr).
\end{align*}

\section{Tree-Structured Diffusion Language Model}
\textbf{Tree and Notation.}
We start by defining a token tree $\mathcal{T}_{\mathrm{token}}$ with nodes $\mathcal{N}$. To accommodate for diffusion process, we assume that all leaf nodes have the same depth, such that,
\begin{equation*}
    \mathcal{N} = \mathcal{L} \cup \left(\bigcup^{H}_{h=1}\mathcal{I}_h\right),
\end{equation*}

where $H$ is the height of the tree $\mathcal{T}_{\mathrm{token}}$, $\mathcal{L}:=\{N_x: x \in \mathcal{V}\}$ is the set of leaf nodes $N_x$, each representing a token $x$ in the vocabulary $\mathcal{V}$, and $\mathcal{I}_h$ contains all the nodes of height $h$.

Given the tree structure, the diffusion process $\{z_t\}_{t\in [0,1]}$ now lives in the finite state space $\mathcal{N}$. In other words, $z_t$ is a random process whose event space is the token tree's set of nodes. 
To align $z_t$ with the tree structure, we introduce level thresholds $\{t_i\}\subset [0,T]$ with $0=t_0<t_1<...<t_H = T = 1$, such that, as $z_t$ moves within $[t_h, t_{h+1}]$, it is strictly confined within $\mathcal{I}_h\cup\mathcal{I}_{h+1}$.

To facilitate notation, we further define several auxiliary functions. 
First, we define an $h$-dependent \textit{ancestor function}  
\begin{equation}
    \Gamma^{h}_{\uparrow}(z_t): \bigcup_{h' \le h}\mathcal{I}_{h'} \rightarrow \mathcal{I}_h,
\end{equation}
which maps any node of height $h'$ smaller than $h$ to their ancestor node of height $h$. When $h' = h$, the ancestor function remains the identity function, i.e., $\Gamma^{h}_{\uparrow}(z_t) = z_t$ if $z_t \in \mathcal{I}_h$.

Moreover, we also define an $h$-dependent \textit{offspring function} 
\begin{equation}
    \Gamma^{h}_{\downarrow}(z_t): 
    \bigcup_{h' \ge  h}\mathcal{I}_{h'} \rightarrow 2^{\mathcal{I}_h},
\end{equation}
where $2^{\mathcal{I}_h}$ is the power set of $\mathcal{I}_h$. Here, $\Gamma^{h}_{\downarrow}(z_t)$ maps any node of height $h'$ greater than $h$ to its descendents of height $h$. When $h' = h$, the offspring function maps the node to the set of its siblings and itself, i.e., for $z_t \in \mathcal{I}_h, h < H$,
$
\Gamma^{h}_{\downarrow}(z_t) = \Gamma^{h}_{\downarrow}\left(\Gamma^{h+1}_{\uparrow}(z_t)\right).
$

The restriction of $z_t$ in $\mathcal{I}_{h}\cup \mathcal{I}_{h+1}$ motivates a novel modeling approach by breaking the entire process $z_t$ into a sequence of \emph{in-level} processes $z_t^{h}$ defined on $t \in [t_{h}, t_{h+1}]$ for $h\in \{0,1,...,H-1\}$. We note that $z_t^h$ is just an alias of $z_t$ on $t \in [t_{h}, t_{h+1}]$ to emphasize that $z_t$ is evolving on the particular level of the token tree $\mathcal{T}_{\mathrm{token}}$. We omit the superscript $h$ when the context is clear without confusion. After defining in-level processes $\{z^h_t\}^{H-1}_{h=0}$, we adopt a CTMC framework to model \emph{each} of these in-level processes $z_t^h$. 

To facilitate definitions for $z_t$, we define the $\mathbf{OneHot}(N): \mathcal{N} \rightarrow \{0,1\}^{|\mathcal{N}|},$ 
that generates the one-hot encoding vector for any given node $N$ in the token tree $\mathcal{T}_{\mathrm{token}}$. In the sequel, $\mathbf{OneHot}(N)$ should be interpreted as a probability mass vector with all the mass concentrated on $N$. 

\subsection{Forward Process}
We define the forward process through the lens of in-level processes, aligned with the tree structure. For a token $x$, the marginal distribution of the intermediate state $z_t^h$ is defined on $t\in [t_h, t_{h+1}]$, for all $h\in \{0,1,...,H-1\}$,
\begin{align*}
    q_{t|t_h}(z_t|\Gamma^{h}_{\uparrow}(x)) 
    = Cat(z_t|\alpha_t^h &\mathbf{OneHot}\left(\Gamma^{h}_{\uparrow}(x)\right)
    +(1-\alpha_t^h)\mathbf{OneHot}\left(\Gamma^{h+1}_{\uparrow}(x)\right),
\end{align*}
where the in-level schedule $\alpha_t^h$ defined on $[t_h, t_{h+1}]$ is decreasing for all $h$, with $\alpha_{t_h}^h = 1$ and $\alpha_{t_{h+1}}^h = 0$.


By construction, $z_{t_h} = \Gamma^{h}_{\uparrow}(x)$ with probability $1$. Therefore, conditioning on $x$'s ancestor $\{z_{t_h} = \Gamma^{h}_{\uparrow}(x)\}$ is equivalent to conditioning on the initial token $\{z_0 = x\}$:

\begin{lemma}\label{lemma1}
The forward process $q_{t|0}(z_t|x)$ on $t\in [t_h,t_{h+1}]$ is equivalent to first mapping $x$ to its ancestor node at height $h$ and then diffusing within that level:
\begin{align*}
q_{t\mid 0}(z_t\mid x)
&= q_{t_h\mid 0}\!\left(\Gamma^{h}_{\uparrow}(x)\mid x\right)\,
   q_{t\mid t_h}\!\left(z_t\mid \Gamma^{h}_{\uparrow}(x)\right) \notag= q_{t\mid t_h}\!\left(z_t\mid \Gamma^{h}_{\uparrow}(x)\right).
\end{align*}
\end{lemma}

\textbf{Generator and cumulative transition matrix.} 

Each in-level process $z_t^h$ follows a GIDD-style interpolation that starts from the initial state $\Gamma^{h}_{\uparrow}(x)$ and mixes toward an absorbing state $\pi_t(x) = \Gamma^{h+1}_{\uparrow}(x)$, with both endpoints depending on $x$. Consequently, each in-level process $z_t^h$ admits a CTMC formulation. We next present their generator and cumulative transition matrices for in-level processes.

\begin{proposition}\label{prop:cum matrix}
The in-level time-inhomogeneous forward transition rate matrix $Q_t$ on $t\in [t_h,t_{h+1})$, and the in-level time-inhomogeneous cumulative conditional transition matrix $P_{t\mid s}$ for $t_h \le s \le t \leq t_{h+1}$ are
\[
Q_t=
\begin{bmatrix}
0 & 0 & 0 & 0\\
0 & \dfrac{{\alpha_t^h}'}{\alpha_t^h}\,I_{|\mathcal{I}_h|} & -\dfrac{{\alpha_t^h}'}{\alpha_t^h}\,\mathbf{\Gamma_{\uparrow}^{(h,h+1)}} & 0\\
0 & 0 & 0 & 0
\end{bmatrix},\ \ 
P_{t\mid s}
=
\begin{bmatrix}
I & 0 & 0 & 0\\[4pt]
0 & \dfrac{\alpha_t^h}{\alpha_s^h}\, I_{|\mathcal{I}_{h}|} 
  & 0
  & 0\\[8pt]
0 & \Bigl(1-\dfrac{\alpha_t^h}{\alpha_s^h}\Bigr)\,\bigl(\mathbf{\Gamma_{\uparrow}^{(h,h+1)}}\bigr)^{\top} 
  & I_{|\mathcal{I}_{h+1}|} 
  & 0\\[4pt]
0 & 0 & 0 & I
\end{bmatrix},
\]
where $\mathbf{\Gamma_{\uparrow}^{(h,h+1)}}\in \mathbb{R}^{|\mathcal{I}_{h+1}|\times|\mathcal{I}_{h}|}$ is a matrix that maps a probability mass distribution of nodes at height $h$ to its corresponding probability mass distribution of nodes at height $h+1$, according to the tree structure.

The general cross-level \(P_{t\mid s}\) then follows from the Markov property: for $t_i\leq s\leq t_{i+1}\leq t_j\leq t\leq t_{j+1}$,
\begin{align*}
P_{t\mid s}
&=
P_{t_{i+1}\mid s}\;
P_{t_{i+2}\mid t_{i+1}}\;
\cdots\;
P_{t_j\mid t_{j-1}}\;
P_{t\mid t_j}.
\end{align*}
\end{proposition}
\begin{proof}[Proof of Proposition~\ref{prop:cum matrix}]
See Appendix~\ref{Appendix:derivation of cum matrix}. 
\end{proof}



We note that, while each in-level process admits a CTMC formulation, the entire process across levels is \emph{not} a CTMC, because its forward rate matrix develops singular points at every level threshold \citep{norris1998markov, campbell2022continuous}. As a result, the standard time-inhomogeneous CTMC interpretation—and thus the related CTMC properties—only holds for each in-level process $z_t^h$ instead of the entire process $z_t$.

\subsection{Reverse Process}
In parallel with the forward construction, we model the reverse process of each in-level process $z_t^h$. For $t_{h-1}<s<t<t_{h}$, we adopt the standard Bayesian parameterization used in discrete diffusion models \citep{austin2021structured, rutte2025generalized} but, leveraging Lemma~\ref{lemma1}, condition on the predicted children distribution of $z_t$, $\mathbf{x}^{h-1}_{\theta}(z_t,t)$, rather than directly predicting the leaf token:
\begin{align}\label{eq:reverse}
p_{\theta}(z_s \mid z_t)
&= q_{t\mid s}(z_t \mid z_s)\,\frac{q_s(z_s \mid \mathbf{x}^{h-1}_{\theta})}{q_t(z_t \mid \mathbf{x}^{h-1}_{\theta})},
\end{align}



where $q_t(\cdot \mid \mathbf{x}_{\theta}^{h-1}) :=\sum_{j:\,x_j \in \Gamma^{h-1}_{\downarrow}(z_t)} \mathbf{x}_{\theta}^{h-1}[j]\; q_t(\cdot \mid x_j)$, $\mathbf{x}^{h-1}_{\theta}(z_t,t)$ is shortened to $\mathbf{x}^{h-1}_{\theta}$ whenever it does not cause confusion, and $\mathbf{x}^{h-1}_{\theta}[j]$ is the predicted probability of $z_t$'s $j$-th child. The general cross-level reverse conditional follows by first factorizing across levels via the Markov property, and then applying the in-level reverse kernel Eq.~\ref{eq:reverse} (See Appendix~\ref{Appendix:general-reverse-kernel}).

\textbf{Backward Rate Matrix}. Since our in-level forward process is an extended instance of GIDD, the corresponding in-level backward rate matrix admits the same closed-form relationship to the forward rate matrix as in GIDD. The only difference is that, we parameterize the reverse process using Eq~\eqref{eq:reverse}: 
\begin{align*}
\hat{Q}_t^{\theta}(z_t,z_s)
&= Q_t(z_s,z_t)\,
  \frac{q_t(z_s \mid \mathbf{x}^{h-1}_{\theta})}{q_t(z_t \mid \mathbf{x}^{h-1}_{\theta})}- \delta_{z_s,z_t}
  \sum_{z'} Q_t(z',z_t)\,
  \frac{q_t(z' \mid \mathbf{x}^{h-1}_{\theta})}{q_t(z_t \mid \mathbf{x}^{h-1}_{\theta})},
\end{align*}
where $\delta_{z_s,z_t}$ is the Kronecker delta function, which equals $1$ if $z_s=z_t$, and $0$ otherwise. For clarity, we omit $\theta$ on $\hat{Q}_t^{\theta}(z_t,z_s)$ whenever the context is clear.

\subsection{Training ELBO}
Across levels, our process forms a Markov chain whose transition rate is singular at every level threshold. Within each level, nonetheless, the process is a well-defined time-inhomogeneous CTMC and inherits standard CTMC properties. We therefore decompose the full trajectory  $z_t$ into in-level processes $z_t^h$ and derive a continuous-time ELBO for each in-level CTMC. Subsequently, by the Markov property, the cross-level transition factors into a product of in-level transitions, yielding an overall training objective given by the summation of the ELBOs over all levels.

\begin{lemma}[Adapted proposition H.4 in \cite{rutte2025generalized}]
\label{lemma:giddelbo}
Let $x^{h} = \Gamma^{h}_{\uparrow}(x)$ and $x^{h+1} = \Gamma^{h+1}_{\uparrow}(x)$. Consider the CTMC diffusion process on the interval $t\in [t_{h}, t_{h+1}]$ with marginal $q_t(z_t \mid x^{h})$, forward rate $Q_t(z_s, z_t)$, backward rate $\hat{Q}_t(z_t, z_s)$, and the reverse process defined in Eq~\ref{eq:reverse}. Then the continuous-time ELBO for each in-level process $z_t^{h}$ satisfies
\begin{align*}
&\log p(z_{t_{h}}=x^{h}|z_{t_{h+1}}=x^{h+1})\geq ELBO(h):= \\
\mathbb{E}_{t, z_t} \Bigg[
\sum_{z_s \neq z_t} Q_t(z_s, z_t)
&\frac{q_t(z_s \mid x^{h})}{q_t(z_t \mid x^{h})} \cdot
\log \frac{q_t(z_s \mid \mathbf{x}^{h}_{\theta}) q_t(z_t \mid x^{h})}{q_t(z_t \mid \mathbf{x}^{h}_{\theta}) q_t(z_s \mid x^{h})}- \sum_{z'} Q_t(z', z_t) \frac{q_t(z' \mid \mathbf{x}^{h}_{\theta})}{q_t(z_t \mid \mathbf{x}^{h}_{\theta})}
\Bigg] + C,
\label{eq:ctmc_elbo}
\end{align*}
where $t \sim \mathcal{U}(t_{h},t_{h+1})$ and 
\begin{align*}
C = \mathbb{E}_{q_{t_{h}}(z_{t_{h}} \mid x^{h})}[\log p(x^{h} \mid z_{t_{h}})] - D_{\mathrm{KL}}(q_{t_{h+1}}(z_{t_{h+1}} \mid x^{h}) \Vert p_{t_{h+1}}(x^{h+1})).
\end{align*}
\end{lemma}

\begin{theorem}[Closed Form In-Level CT-ELBO of TDLM]
\label{theo:in-level ELBO}
Following the notation in Lemma~\ref{lemma:giddelbo}, the continuous-time ELBO for in-level process $z_{t}^{h}$ admits a closed form:
\begin{align*}
\mathrm{ELBO(h)}
&= \mathbb{E}_{t,z_t}\!\Bigg[
\mathbf{1}\left\{ z_t = x^{h+1} \right\} \left(-\frac{{\alpha_t^h}'}{1-\alpha_t^h}\right)
\log p_{\theta}^{h}\!\left(x^{h}\right)\Bigg]
,
\end{align*}
where $t \sim \mathcal{U}[t_{h},t_{h+1}]$ and $p_{\theta}^{h}$ is the predicted probability mass function over ${z_t}'s$ children nodes. 
\end{theorem}
\begin{proof}[Proof of Theorem~\ref{theo:in-level ELBO}]
    See Appendix~\ref{appendix:proof in-level ELBO}.
\end{proof}

Theorem~\ref{theo:in-level ELBO} implies that the training objective reduces to predicting the ground-truth child node from the current state whenever the state is not yet transitioned within the level. Consequently, our framework classifies children of the current node instead of the entire set of vocabulary, substantially reduces memory and computation, and enables promising architectural and algorithmic opportunities discussed in Sections~\ref{TDLMDiscussion} and Section~\ref{sec:discussion}. A closed-form cross-level ELBO for token $x$ follows from the Markov property together with the tree structure.

\begin{corollary}[Closed-form cross-level ELBO of TDLM] 
\label{theorem:cross-level-elbo}
The cross-level ELBO of TDLM equals the summation of all in-level ELBOs, i.e.,
\begin{align*}
\log p(x)  &= \log\sum_{z_{t_1},...,z_{t_H}}\prod_{h=1}^H p(z_{t_{h-1}}|z_{t_h}) p(z_{t_H})\\& = \log\prod_{i=1}^Hp(z_{t_{h-1}}=\Gamma^{h-1}_{\uparrow}(x)|z_{t_h} = \Gamma^{h}_{\uparrow}(x))\geq\sum_{h=1}^H ELBO(h)
\end{align*}
\end{corollary}





\subsection{Parameter and Memory Efficiency}
\label{TDLMDiscussion}
A key advantage of TDLM is the parameter efficiency. Standard token prediction uses an output projection of size $d \times V$, where $d$ is the model dimension and $V$ is the vocabulary size. TDLM instead predicts children in a vocabulary tree, reducing the output layer to size $d \times K$, where $K$ is the branching factor and is typically a small constant independent of $V$. For example, a two-level tree with $K=512$ can represent over $250{,}000$ tokens. When $d=768$, $V\approx 50{,}000$, and $K=512$, the output layer shrinks from $38.4$ million to $0.4$ million parameters, a nearly $100\times$ reduction. In small-scale DiT-style models, this output layer can account for over $20\%$ of total parameters, so the savings can be reallocated to the backbone.

TDLM is also substantially more memory efficient during training. In standard token prediction, the output logits have shape $B \times S \times V$, where $B$ denotes the batch size, $S$ the sequence length, and $V$ the vocabulary size, so activation memory scales linearly with $V$. Under BF16 with $B=512$, $S=512$, and $V \approx 50{,}000$, the logits alone require about $24.4$~GiB of memory. With tree-based prediction and $K=512$, this is reduced to about $0.25$~GiB. Since training also materializes transient tensors of comparable size, the practical memory savings are even greater. Empirically, TDLM reduces GPU memory usage by roughly half relative to competing methods using small- and base-scale DiT models, making it particularly well suited to resource-constrained settings.


\section{Experiments}
\subsection{Experimental Setup}
We follow the experimental setup of prior works~\cite{rutte2025generalized, zhou2025next} to evaluate the language modeling capability of our model. Specifically, we use the widely adopted OpenWebText (OWT) dataset with sequence length $512$ and no packing. We take $100000$ out of $8,013,769$
data as the validation set. Architecturally, our approach yields a much smaller classification head, reducing parameters by about $23\%$ ($12\%$) compared to DiT-small (DiT-base). To maintain a comparable total parameter budget, we increase the number of attention blocks. Specifically, while DiT-small uses $12$ blocks ($92.1$M parameters) with large embedding/classification layers ($77.8$M), our model uses $17$ blocks ($130.5$M) with substantially smaller embedding/classification layers ($39.7$M). Likewise, for DiT-base, we increase the number of attention blocks from $24$ ($321.2$M) to $27$ ($361.3$M), while reducing the embedding/classification layers from $103.6$M to $52.9$M to maintain comparable or smaller total parameters to prior works. Implementation details are provided in the Appendix~\ref{sec:imple-detail}.

To construct the vocabulary tree, we apply a recursive K-means procedure \citep{zhou2025next} that partitions each node into $K$ children (branching factor) until each leaf contains a single token  (see Algorithm \ref{alg:tree_construct}). To accommodate diffusion modeling, we enforce a uniform leaf depth by padding shorter paths in the tree, repeating the terminal token until the maximum depth is reached. A prescribed min/max ratio is used as a hyper-parameter to control the range of node size. For experiments in section~\ref{sec:main-results}, we use $K = 512$ and min/max ratio $0.8/1.2$. We do a full ablation study on the choice of these hyperparameters in Section~\ref{sec:ablation}.

\begin{wraptable}{r}{0.58\textwidth}
\vspace{-10pt}
\centering
\footnotesize
\setlength{\tabcolsep}{2pt}
\renewcommand{\arraystretch}{1.0}

\resizebox{0.6\textwidth}{!}{%
\begin{tabular}{l c c c}
\toprule
\textbf{Model} & \textbf{Train. toks.} & \textbf{Val. PPL ($\downarrow$)} & \textbf{Gen. PPL ($\downarrow$)} \\
\midrule
GPT2~\textsuperscript{$\dagger$} & unk. & 23.40 & -- \\
Llama110M (retrain.)\textsuperscript{$\dagger$} & 262B & 16.11 & -- \\
\midrule
SEDD \citep{lou2023discrete}~\textsuperscript{$\dagger$} & 262B & $\le$24.10 & -- \\ 
\hdashline
MDLM-small \citep{sahoo2024simple}\textsuperscript{$*$} & 131B & $\le$27.39 & 163.7 \\
GIDD+-small \citep{rutte2025generalized}\textsuperscript{$*$} & 131B & $\le$25.82 & 170.2 \\
HDLM-small \citep{zhou2025next}\textsuperscript{$*$} & 131B & $\le \mathbf{23.25}$ & \textbf{148.0} \\
TDLM-small (Ours) & 131B & $\le \underline{25.50}$ & \underline{159.3} \\
\midrule
HDLM-base \citep{zhou2025next}\textsuperscript{$*$} & 131B & $\le \underline{19.22}$ & \underline{139.9} \\
TDLM-base (Ours) & 131B & $\le \mathbf{18.95}$ & $\mathbf{138.0}$ \\
\bottomrule
\vspace{-8pt}
\end{tabular}}

\caption{Validation and generative perplexity on OWT. \textbf{Bold} and \underline{underline} denote best and second best. Adopted from \citet{zhou2025next}\textsuperscript{$*$}; adopted from corresponding previous work\textsuperscript{$\dagger$}.}
\label{tab:main_result}
\vspace{-20pt}
\end{wraptable}

\subsection{Main Results}

\label{sec:main-results}
\begin{figure*}[htbp]
\vspace{-10pt}
  \centering
  \includegraphics[width=0.49\linewidth]{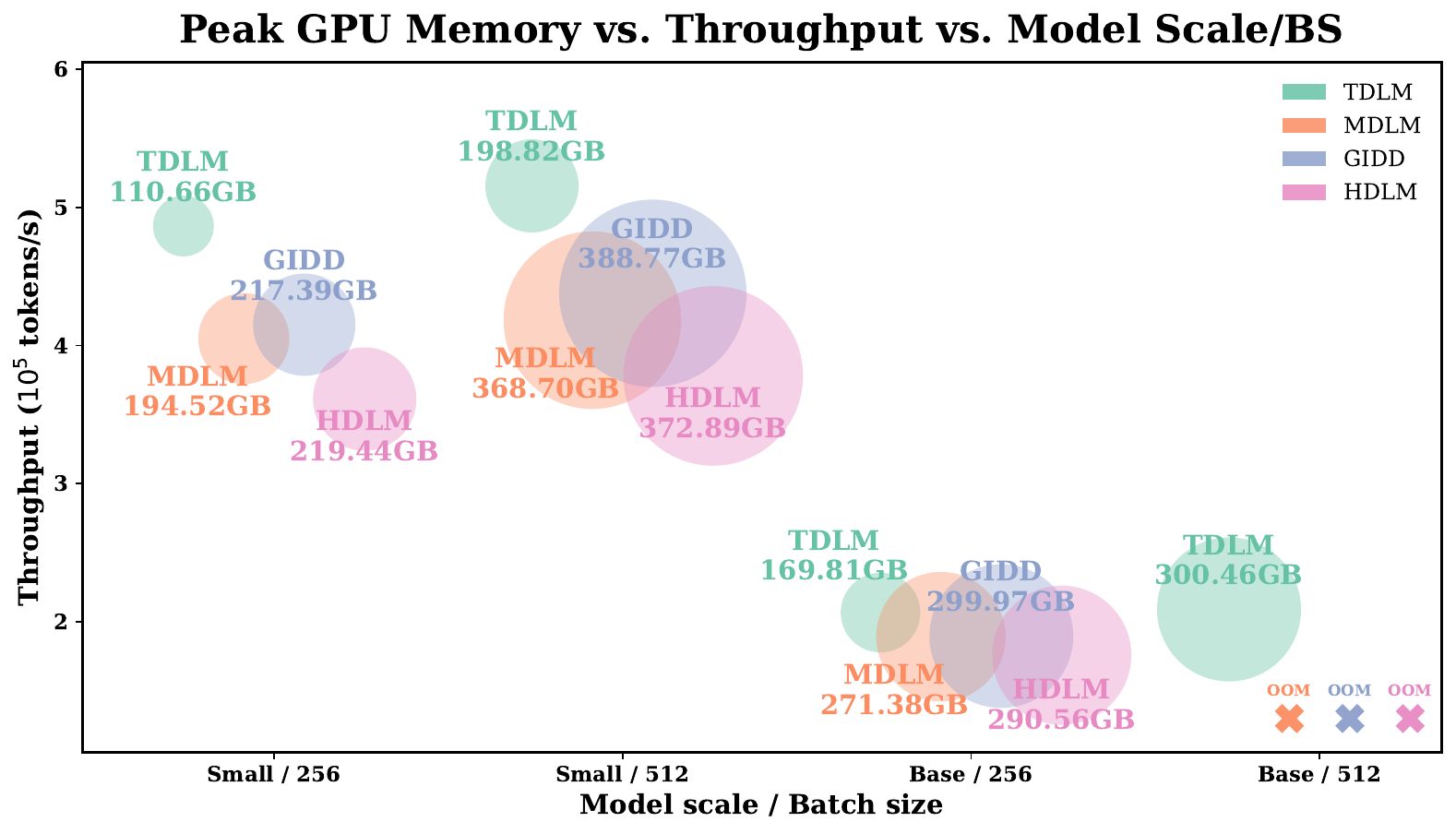}
  \includegraphics[width=0.49\linewidth]{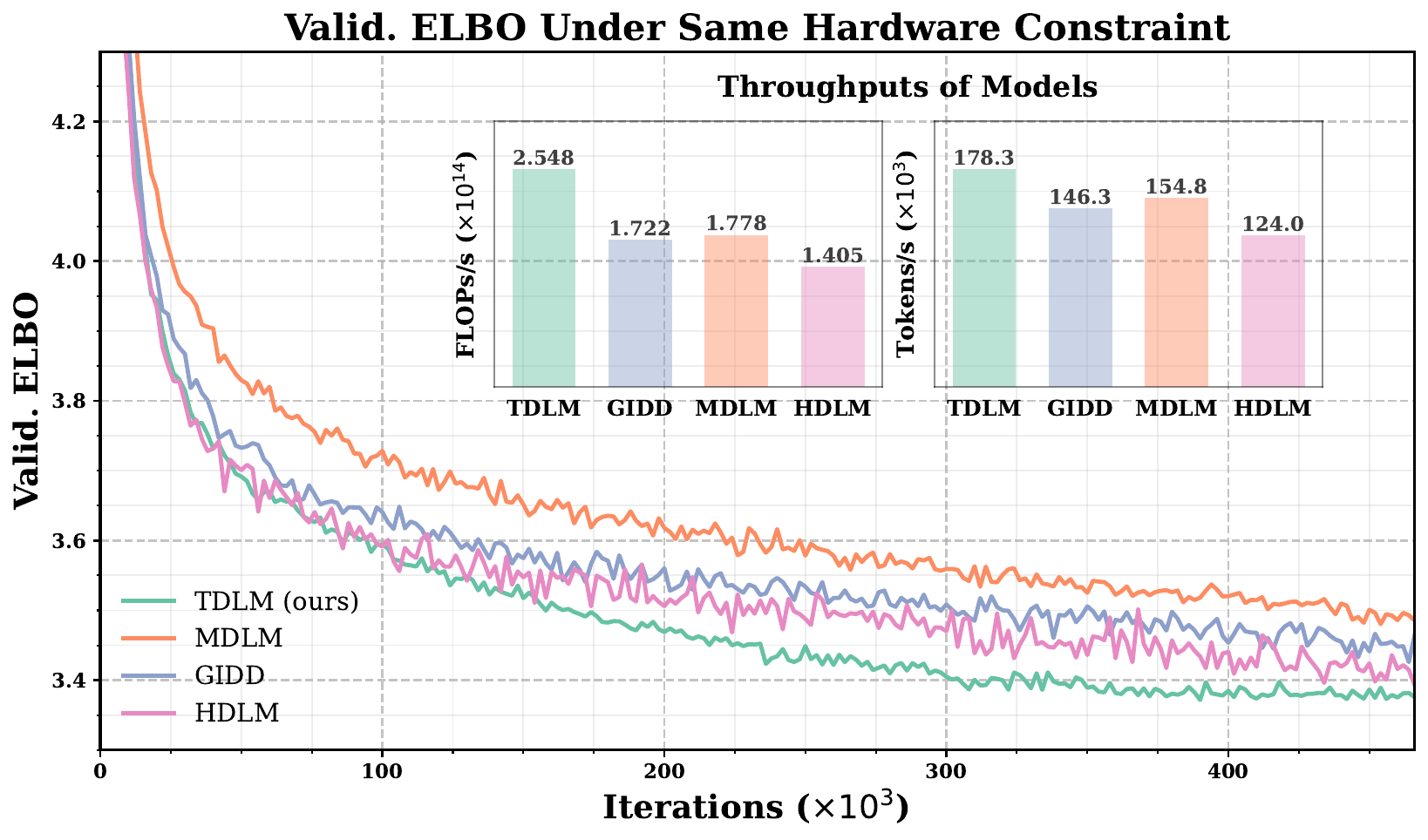}
  \caption{Memory effiency of TDLM. (Left) Throughput and peak memory comparison (denoted as size of the circle) across diffusion language models at different scales. (Right) Validation negative ELBO and average Flops/s of small scale models during training using full capacity of four 24 GB GPUs.}
  \label{fig:hardware-ablation}
  \vspace{-10pt}
\end{figure*}
Following previous works~\cite{zhou2025next}, we adopt the validation perplexity and the generative perplexity as metrics, using gpt2-large as the reference model. We include autoregressive models and contemporary diffusion language models as baselines to benchmark our model.

Table \ref{tab:main_result} shows that our method outperforms MDLM and GIDD on both perplexity metrics, supporting the benefit of exploiting tree-structured token hierarchy, and suggesting that child prediction is a viable language modelling solution. Compared with HDLM, our method is less competitive at small scale but slightly surpasses it at base scale. This trend suggests that base models better capture per-level dynamics, allowing improvements at each level to accumulate into stronger overall performance.

Our method also substantially improves training efficiency, since child prediction operates over a much smaller label space than full-vocabulary classification. As shown in Figure \ref{fig:hardware-ablation} (left), at matched model scale and batch size, it roughly halves peak GPU memory on both scales and improves training throughput by about $25\%$ for small models. Under a limited budget of four RTX-3090 GPUs, Figure \ref{fig:hardware-ablation} (right) further demonstrates that our method achieves both higher throughput and lower validation perplexity than all baselines. These results highlight improved training efficiency and perplexity performance under resource-constrained settings.


\subsection{Ablation Studies}
\label{sec:ablation}
In this section, we attempt to answer the following questions: (1) How do the branching factor, height, and the node size of the tree affect performances? (2) How does the validation ELBO change across levels? (3) How does assigning different training weights to each level affect model performance? (4) How does different sampling schedule affect the generative perplexity? All ablations are completed with small-scale model and batch size $128$.

\textbf{Branching Factor, Height, and Cluster Size.} Figure~\ref{fig:ablation-figures} (left) presents two sets of ablations that jointly probe tree-construction hyperparameters. In the first ablation (solid curves), we fix the node-size ratio and vary the branching factor $K$, which accordingly determines the resulting tree height $H$. Across these settings, we observe a consistent pattern that shallower trees tend to achieve lower validation negative ELBO. The effect of $K$ itself, however, is mixed once tree height is fixed: within the $H{=}3$ group ($K\in\{64,128,256\}$) and the $H{=}2$ group ($K\in\{512,1024\}$), the curves are distinguishable but not strictly monotonic.

\begin{figure}[t]
  \centering

  \begin{subfigure}[t]{0.52\linewidth}
    \centering
    \includegraphics[width=\linewidth]{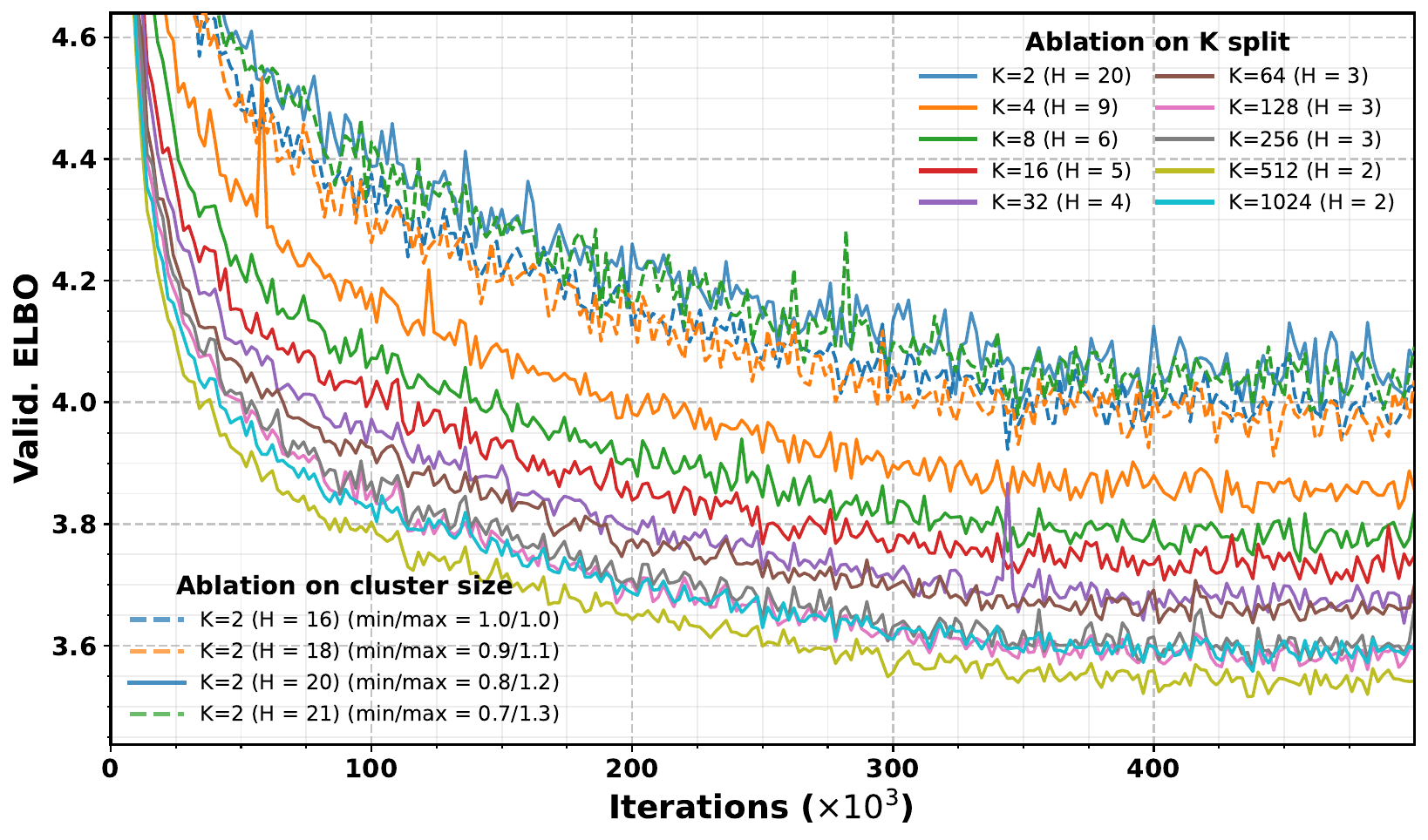}
    \label{fig:k-split-clustersize-ablation}
  \end{subfigure}
  \hfill
  \begin{subfigure}[t]{0.47\linewidth}
    \centering
    \includegraphics[width=\linewidth]{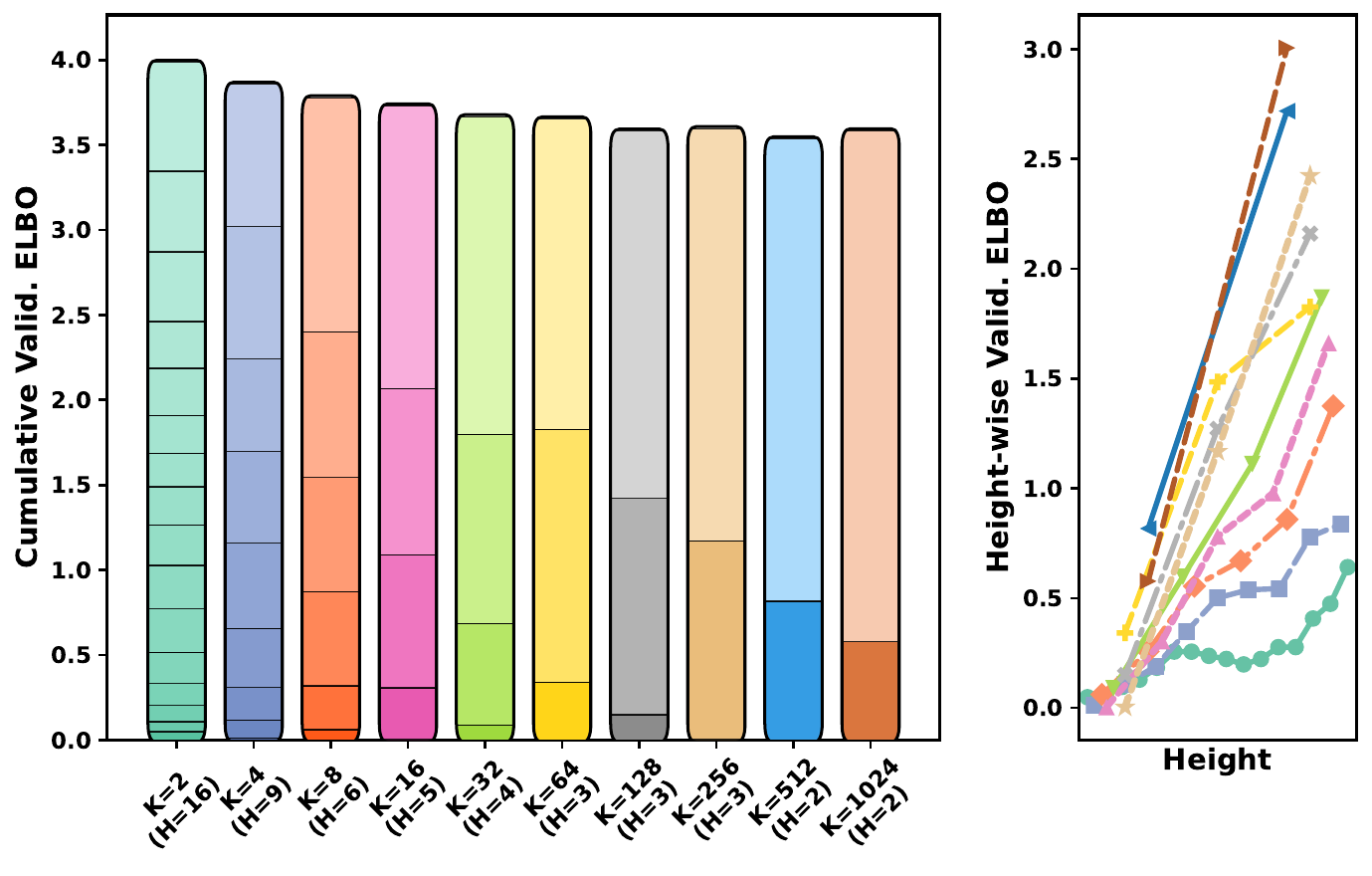}
    \label{fig:level-elbo-ablation}
  \end{subfigure}
\vspace{-20pt}
  \caption{Left: Validation nagative ELBO of various branching factor $K$ and cluster size ratio, plotted over training iterations. Solid lines indicate varying $K$ and fixed cluster size ratio, while dashed lines indicate fixed $K$ but varying cluster size. Middle: each bar indicates the cumulative validation negative ELBO obtained from a different tree construction, while each stack indicates the contribution of a particular tree height. (top stack indicates root level) Right: raw validation negative ELBO of each height level under different tree constructions.
}
  \label{fig:ablation-figures}
  \vspace{-15pt}
\end{figure}

In the second ablation (dashed curves), we fix $K{=}2$ and vary the node-size ratio. Varying the node-size ratio increases the height from $H{=}16$ to $H{=}21$ as child nodes become more imbalanced.
These curves remain tightly clustered throughout training, suggesting a comparatively smaller effect on ELBO than the changes incurred by the branching factor.

\textbf{Pattern of In-Level ELBOs.} By Theorem~\ref{theorem:cross-level-elbo}, TDLM’s cross-level ELBO decomposes into a sum of in-level ELBOs, making it important to understand the behavior of each level. Figure~\ref{fig:ablation-figures} (middle and right) shows that these terms grow roughly linearly with height for most trees, but nearly exponentially for the binary tree due to its much greater depth. 
As a result, higher levels near the root contribute a large share of the cross-level ELBO, consistent with the intuition that coarser, less informative levels are harder to model. It remains unclear, however, whether this imbalance reflects suboptimal optimization or an inherent property of the tree structure. 
To investigate this, we introduce level-specific weights in the training objective and study their effect on both in-level and cross-level ELBOs.

\textbf{Level-wise Training Weights.}
We investigate whether placing greater emphasis on higher hierarchical levels during training can reduce their in-level ELBOs.
In the binary-tree setting, where this disparity is most severe, we apply level-wise weights that increase with height, using either a linear or exponential schedule, parameterized by $\gamma_{\mathrm{lin}}$ and $\gamma_{\mathrm{exp}}$ (see Appendix~\ref{section:weight-schedule}). These schemes impose heavier penalties on higher levels, which are more challenging and information-sparse. As shown in Figure~\ref{fig:weight-ablation}, however, the gains are marginal: only the mildest linear reweighting yields a slight improvement, whereas stronger reweighting consistently degrades performance. 
Notably, the exponential schedule with $\gamma_{\mathrm{exp}}=0.3$, which most closely follows the empirical growth of the in-level ELBOs, yields the worst validation cross-level ELBO. This suggests that the time-dependent model is already near-optimal under the unweighted ELBO, and that level-wise reweighting mainly distorts the target objective.
\begin{wrapfigure}{r}{0.5\textwidth}
\vspace{-8pt}
  \centering
  \includegraphics[width=\linewidth]{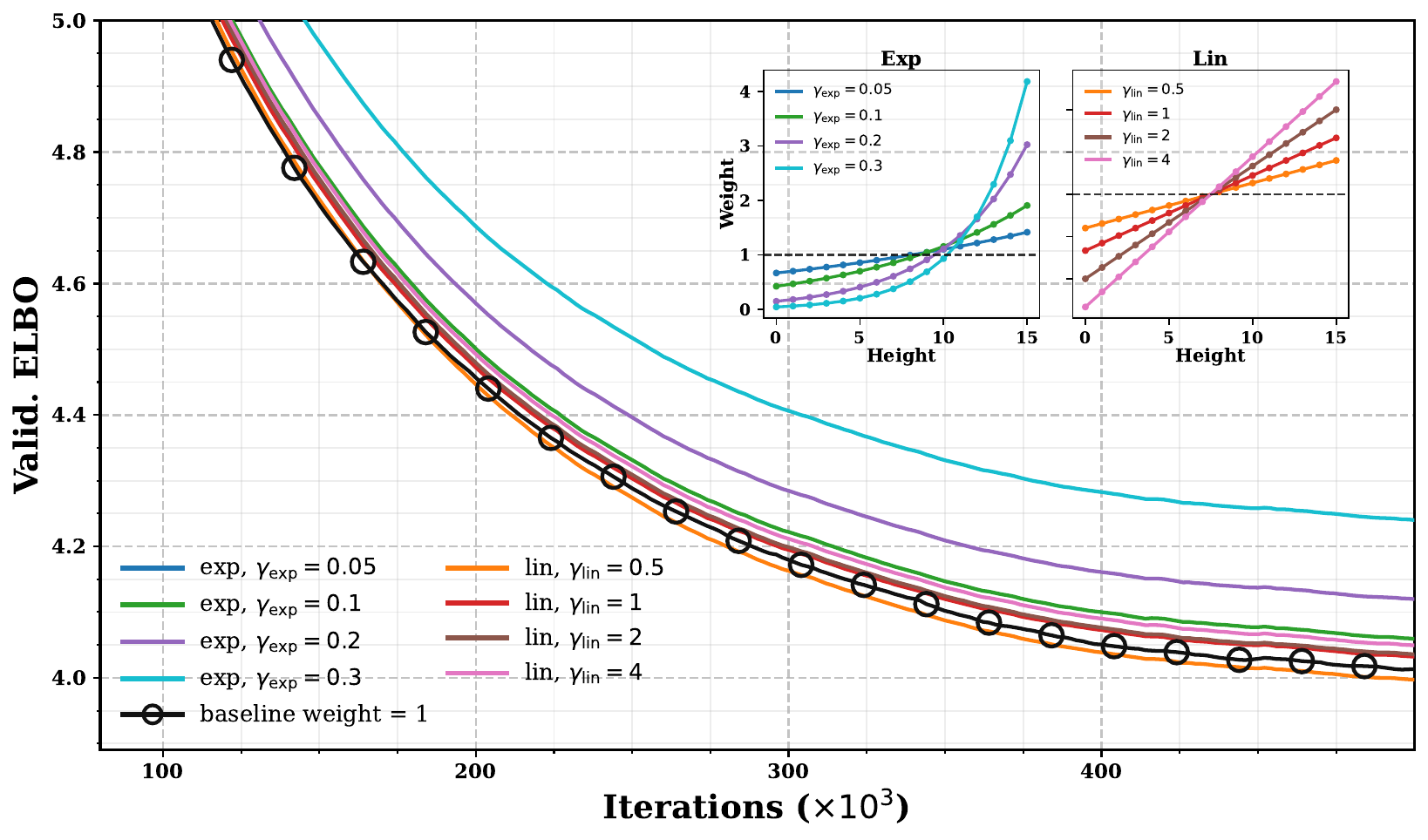}
  \vspace{-15pt}
  \caption{Smoothed validation negative ELBO of different weight schedules. The weight schedules are plotted in the subfigures (Left: exponential; Right: linear).}
  \label{fig:weight-ablation}
\vspace{110pt}
\end{wrapfigure}

\begin{wraptable}{r}{0.5\textwidth}
\vspace{-385pt}
  \centering
  \resizebox{\linewidth}{!}{%
  \begin{tabular}{ccc}
    \toprule
    Model & \makecell{Inference Steps\\(Top/Bottom Level)} & Gen.\ PPL ($\downarrow$) \\
    \midrule
    \multirow{7}{*}{\makecell{TDLM-base\\($H=2$)}}
      &  64 / 448 & 142.1 \\
      & 128 / 384 & 140.3 \\
      & 192 / 320 & 140.7 \\
      & 256 / 256 & 138.0 \\
      & 320 / 192 & 141.7 \\
      & 384 / 128 & 149.2 \\
      & 448 / 64 & 165.3 \\
    \bottomrule
  \end{tabular}}
  \caption{Generation perplexity with different allocations of the total 512 inference steps using TDLM-base model.}
  \label{tab:tdlm_base_inference_steps_genppl}
\end{wraptable}

\textbf{Sampling Schedule.} 
Another unique perspective of our approach is that its inference process can be viewed as multiple independent inference processes. 
Given the observation that higher levels incur larger in-level ELBOs, it seems favorable to allocate more inference steps to those levels. 
However, using the two-level TDLM base model and fixing the total number of inference steps, we empirically find that a balanced allocation of inference steps between the two levels yields the best result as shown in table~\ref{tab:tdlm_base_inference_steps_genppl}.

\section{Related Works}
In the pre-transformer era, structured vocabulary is explored to enable subgroup prediction in autoregressive models, as exemplified by hierarchical softmax \citep{morin2005hierarchical} and adaptive softmax \citep{grave2017efficient}, primarily for computational efficiency. However, hierarchical formulations are not naturally well aligned with autoregressive modeling, and direct prediction over a flat vocabulary has remained the dominant design in modern language models. As language models continue to scale, training is increasingly constrained by GPU memory capacity. To alleviate this burden, a range of methods has been proposed, including quantization to reduce numerical precision~\citep{zhu2024survey}, low-rank gradient projection~\citep{zhao2024galore}, and low-rank activation projection~\citep{shamshoum2025compact}. Despite their effectiveness, these approaches mainly target optimizer states or intermediate representations, and do not directly reduce the activations of the output layer, which remain a major source of training-time memory consumption in smaller-scale models. With the emergence of diffusion language models, interest in exploiting vocabulary structure has reemerged. Recent works, using approaches such as semantic clustering \citep{zhou2025next} and $n$-bit token representations \citep{chao2025beyond}, revisit structured vocabulary in diffusion language models, where hierarchical formulations align more naturally with the coarse-to-fine denoising process. Nonetheless, these methods still predict over the flat vocabulary and are primarily motivated by improved perplexity rather than parameter or memory efficiency. In contrast, our work formally formulates a tree-structured vocabulary, replaces direct flat-vocabulary prediction with children prediction, and improves the efficiency of calssification head in diffusion language models while achieving performance comparable to the state of the art.
\section{Conclusion and Discussion}
\label{sec:discussion}
In this work, we propose a tree-structured discrete diffusion language model that parameterizes the posterior via child-prediction instead of token-level prediction. Under resource-constrained settings, this formulation substantially improves memory and parameter efficiency, while achieving performance comparable to the state-of-the-art.

Beyond efficiency and performance gains, child prediction enables new algorithmic and architectural opportunities. In particular, the reduced prediction space makes joint modeling tractable; we briefly discuss this direction and provide preliminary results in Appendix~\ref{appendix:joint-modeling}. Moreover, the parameters saved in the output layer can be reallocated to expand tokenizer's vocabulary, allowing for finer-grained textual representations at a fixed sequence length.

\bibliography{colm2026_conference}
\bibliographystyle{colm2026_conference}

\newpage
\appendix
\section{Proofs and Derivations}
\subsection{Proof of Theorem~\ref{theo:in-level ELBO}} 
\label{appendix:proof in-level ELBO}Here we restate the theorem. 
\begin{theorem}[Closed Form In-Level CT-ELBO of TDLM]
Following the notation in Lemma~\ref{lemma:giddelbo}, the continuous-time ELBO for in-level process $z_{t}^{h}$ admits a closed form:
\begin{align*}
\mathrm{ELBO(h)}
&= \mathbb{E}_{t,z_t}\!\Bigg[
\delta\left\{ z_t = x^{h+1} \right\} 
\left(-\frac{{\alpha_t^h}'}{1-\alpha_t^h}\right)
\log p_{\theta}^{h}\!\left(x^h\right)\Bigg]
,
\end{align*}
where $t \sim \mathcal{U}[t_{h},t_{h+1}]$ and $p_{\theta}^{h}$ is the model-predicted probability mass function over ${z_t}'s$ children nodes. 
\end{theorem}

\begin{proof}
    
\begin{align*}
&\log p(z_{t_{h}}^{h}=x^{h}|z_{t_{h+1}}^{h}=x^{h+1})
\\&\ge
\mathbb{E}_{t,z_t}
\Biggl[
\sum_{z_s \neq z_t}
Q_t(z_s,z_t)\,
\frac{q_t(z_s \mid x)}{q_t(z_t \mid x)}
\Biggl\{
\log
\frac{\hat{Q}_t(z_t,z_s)\,q_t(z_t \mid x)}
     {Q_t(z_s,z_t)\,q_t(z_s \mid x)}\Biggr\}
+ \hat{Q}_t(z_t,z_t)
- Q_t(z_t,z_t)
\Biggr]
+ C\\
&= \mathbb{E}_{t,z_t} \Biggl[
   \underbrace{\sum_{z_s \ne z_t}
   Q_t(z_s, z_t)\,
   \frac{
     q_t\bigl(z_s \mid x^h\bigr)
   }{
     q_t\bigl(z_t \mid x^h\bigr)
   }
   \log
   \frac{
     q_t\bigl(z_s \mid \mathbf{x}_\theta^{h}\bigr)\,
     q_t\bigl(z_t \mid x^h\bigr)
   }{
     q_t\bigl(z_t \mid \mathbf{x}_\theta^{h}\bigr)\,
     q_t\bigl(z_s \mid x^h\bigr)
   }}_{\text{First Term}}
\underbrace{- \sum_{z'}
   Q_t(z', z_t)\,
   \frac{
     q_t\bigl(z' \mid \mathbf{x}_\theta^{h}\bigr)
   }{
     q_t\bigl(z_t \mid \mathbf{x}_\theta^{h}\bigr)
   }}_{\text{Second Term}}
   \Biggr]
   + C,
\end{align*}
Note that $Q_t$ is only non–zero when
\[
z_s = z_t
\quad\text{or}\quad
\bigl(z_t = x^{h+1}
\;\text{ and }\;
z_s = x^h\bigr).
\]

\paragraph{First term in ELBO.}

Since the summation is over $\{z_s\neq z_t\}$, the summand of the first term is only non–zero when
\[
z_t = x^{h+1}\text{ and } 
z_s = x^h.
\]
Therefore,
\begin{align*}
\text{First Term}
&=
-\frac{{\alpha_t^h}'}{\alpha_t^h}\,\frac{\alpha_t^h}{1-\alpha_t^h}\,
\log
\frac{
\displaystyle
\sum_{z \in \Gamma_{\downarrow}^{\,h}(z_t)}
p_{\theta}^{h}(z)\,q_t(z_s \mid z)\,(1-\alpha_t^h)
}{
\displaystyle
\sum_{z \in \Gamma_{\downarrow}^{\,h}(z_t)}
p_{\theta}^{h}(z)\,q_t(z_t \mid z)\,\alpha_t^h
}
\\
    &= -\,\frac{{\alpha_t^h}'}{\alpha_t^h}
      \cdot \frac{\alpha_t^h}{1 - \alpha_t^h}
      \cdot \log
      \frac{
        p_{\theta}^{h}\!\left(x^h\right)
        \alpha_t^h (1 - \alpha_t^h)
      }{
        (1-\alpha_t^h) \cdot 1 \cdot \alpha_t^h
      }\\
&=
-\frac{{\alpha_t^h}'}{1-\alpha_t^h}\,
\log p_{\theta}^{h}\!\bigl(x^h\bigr).
\end{align*}

\paragraph{Second term in ELBO.}

The summand of the second term is non–zero when 

\begin{align*}
    \text{(1) }z' = z_t = x^h\text{, or (2) }z_t = x^{h+1}\text{ and }z' \in \Gamma_{\downarrow}^{\,h}(z_t).
\end{align*}

For Case (1):
\begin{align*}
\text{Second Term}
&=-
\frac{{\alpha_t^h}'}{\alpha_t^h}.
\end{align*}

For Case (2):
\begin{align*}
    \text{Second Term} \quad
    &= -\left( -\frac{{\alpha_t^h}'}{\alpha_t^h} 
    \cdot 
    \frac{
        \left( 
            \sum_{z' \in \Gamma_{\downarrow}^{\,h}(z_t)} 
            p^h_{\theta}(z')\, \alpha_t^h
        \right)
    }{1 - \alpha_t^h}
    \right)\\
    &=\frac{{\alpha_t^h}'}{1-\alpha_t^h}
\end{align*}

\paragraph{Collecting terms.}

Therefore,
\begin{align*}
\text{ELBO}
&=
\mathbb{E}_{t,z_t}
\Bigl[
\delta\{z_t = x^{h+1}\}
\Bigl\{
-\frac{{\alpha_t^h}'}{1-\alpha_t^h}
\log p^{h}_{\theta}\!\bigl(x^h\bigr)
+ \frac{{\alpha_t^h}'}{1-\alpha_t^h}
\Bigr\}
\\
&\hspace{6em}
+ \delta\{z_t = x^h\}
\Bigl\{-\frac{{\alpha_t^h}'}{\alpha_t^h}\Bigr\}
\Bigr].
\end{align*}

The collection of terms irrelevant to $\theta$ has zero expectation:
\begin{align*}
\mathbb{E}_{t,z_t}
\Bigl[
\delta\{z_t &= x^{h+1}\}
\Bigl(\frac{{\alpha_t^h}'}{1-\alpha_t^h}\Bigr)
+ \delta\{z_t = x^h\}
\Bigl(-\frac{{\alpha_t^h}'}{\alpha_t^h}\Bigr)
\Bigr]\\
&=
\mathbb{E}_t\bigl[(1-\alpha_t^h)\tfrac{{\alpha_t^h}'}{1-\alpha_t^h} + \alpha_t^h(-\tfrac{{\alpha_t^h}'}{\alpha_t^h})\bigr]
\\
&=
\mathbb{E}_t\bigl[{\alpha_t^h}' - {\alpha_t^h}'\bigr] = 0.
\end{align*}

Hence
\begin{align*}
\text{ELBO}
&=
\mathbb{E}_{t,z_t}
\Bigl[\delta\{z_t = x^{h+1}\}
(-\frac{{\alpha_t^h}'}{1-\alpha_t^h})\,
\log p^{h}_{\theta}\!\bigl(x^h\bigr)
\Bigr],
\end{align*}

\end{proof}

\subsection{Derivation of Proposition~\ref{prop:cum matrix}}
\label{Appendix:derivation of cum matrix}
Here we restate the proposition.
\begin{proposition}{}
Given the definition of the in-level time-inhomogeneous forward transition rate matrix $Q_t$ on $t\in [t_h,t_{h+1})$, the in-level time-inhomogeneous cumulative conditional transition matrix on $t_h \le s \le t \leq t_{h+1}$ is 
\begin{align*}
P_{t\mid s}
&=
\begin{bmatrix}
I & 0 & 0 & 0\\[4pt]
0 & \dfrac{\alpha_t^h}{\alpha_s^h}\, I_{|\mathcal{I}_{h}|} 
  & 0
  & 0\\[8pt]
0 & \Bigl(1-\dfrac{\alpha_t^h}{\alpha_s^h}\Bigr)\,\bigl(\mathbf{\Gamma_{\uparrow}^{(h,h+1)}}\bigr)^{\top} 
  & I_{|\mathcal{I}_{h+1}|} 
  & 0\\[4pt]
0 & 0 & 0 & I
\end{bmatrix},
\end{align*}
The general \(P_{t\mid s}\) then follows from the Markov property: for $t_i\leq s\leq t_{i+1}\leq t_j\leq t\leq t_{j+1}$,
\begin{align*}
P_{t\mid s}
&=
P_{l_{i+1}\mid s}\;
P_{l_{i+2}\mid l_{i+1}}\;
\cdots\;
P_{l_j\mid l_{j-1}}\;
P_{t\mid l_j}.
\end{align*}
\end{proposition}

\begin{proof}
    Now we derive the conditional cumulative transition matrix $P_{t\mid s}$.
Let $t_h\leq s \le t \leq t_{h+1}$. 
\paragraph{Case 1: } Suppose $x_s = \Gamma_{\uparrow}^{h+1}(x)$. Then $x_s$ is already transmitted into the absorbing state of the level, and therefore with probability $1$ it stays in its current state.
\paragraph{Case 2: } Suppose $x_s = \Gamma_{\uparrow}^{h}(x)$.

The leaving rate of the state $x_s$ is
\[
\lambda_{x_s}(s) = - Q_s(x_s, x_s).
\]

The probability of staying in the state from $s$ to $t$ is:
\[
\Pr(\text{stay in } x_s \text{ on }[s,t]) = \exp\Bigl(-\int_s^t \lambda_{x_s}(u)\, du\Bigr).
\]

\paragraph{Case 2.1:} Suppose $x_t = x_s = \Gamma_{\uparrow}^{h}(x)$.

\begin{align}
P_{t\mid s}(x_t, x_s)
  &= \Pr(\text{stay in } x_s \text{ on }[s,t]) \\
  &= \exp\Bigl(\int_s^t \frac{{\alpha_u^h}'}{\alpha_u^h}\, du\Bigr)
   = \frac{\alpha_t^h}{\alpha_s^h}.
\end{align}

\paragraph{Case 2.2:} Suppose $x_t = \Gamma_{\uparrow}^{h+1}(x)$ and $x_s = \Gamma_{\uparrow}^{h}(x)$.

\begin{align}
P_{t\mid s}(x_t, x_s)
  &= \int_s^t
      \Pr\bigl(\text{stays in }\Gamma_{\uparrow}^{h}(x)
               \text{ from } s \text{ to } u\bigr)\,
      \lambda_{x_s}(u)\,
      \Pr\bigl(\text{stays in }\Gamma_{\uparrow}^{h+1}(x)
               \text{ from } u \text{ to } t\bigr)\, du \\
  &= \int_s^t
      \exp\Bigl(\int_s^u \frac{{\alpha_r^h}'}{\alpha_r^h}\, dr\Bigr)
      \Bigl(-\frac{{\alpha_u^h}'}{\alpha_u^h}\Bigr)\, 1\, du \\
  &= \int_s^t -\frac{\alpha_u^h}{\alpha_s^h}\, \frac{{\alpha_u^h}'}{\alpha_u^h} du \\
  &= \int_s^t -\frac{{\alpha_u^h}'}{\alpha_s^h}\, du \\
  &= -\frac{\alpha_t^h - \alpha_s^h}{\alpha_s^h}\\
   &= 1 - \frac{\alpha_t^h}{\alpha_s^h}.
\end{align}

The above cases compose the stated conditional transition matrix within each level. The general $P_{t|s}$ then follows from the markov chain property of our framework.
\end{proof}

\subsection{Derivation of General Reverse Transition Kernel} \label{Appendix:general-reverse-kernel}
Define $i$ and $j$ by the chain of thresholds
\[
h_{j-1} < s \le h_j \le \cdots \le h_{i-1} < t \le h_i .
\]
We first factorize over the intermediate threshold states using the Markov property:
\begin{align}
p_\theta\!\big(z_s, z_{h_{i-1}}, z_{h_{i-2}},\ldots,z_{h_j} \mid z_t\big)
&=
p_\theta(z_s \mid z_{h_j})
\left(\prod_{k=j}^{i-2} p_\theta(z_{h_k}\mid z_{h_{k+1}})\right)
p_\theta(z_{h_{i-1}} \mid z_t).
\label{eq:threshold-factorization}
\end{align}
Marginalizing over all intermediate threshold states then gives the general reverse transition:
\begin{align}
p_\theta(z_s \mid z_t)
&= \sum_{z_{h_{i-1}},\ldots,z_{h_j}}
p_\theta\!\big(z_s, z_{h_{i-1}},\ldots,z_{h_j} \mid z_t\big)
\label{eq:threshold-marginalization}\\
&= \sum_{z_{h_{i-1}},\ldots,z_{h_j}}
p_\theta(z_s \mid z_{h_j})
\left(\prod_{k=j}^{i-2} p_\theta(z_{h_k}\mid z_{h_{k+1}})\right)
p_\theta(z_{h_{i-1}} \mid z_t).
\label{eq:general-reverse-transition}
\end{align}
This implies that the cross-level reverse process first follows the predicted child mappings up to the nearest ancestor, and then applies the in-level reverse kernel starting from that predicted ancestor.

\section{Implementation Details}
\label{sec:imple-detail}
We adopt the DiT architecture~\citep{peebles2023scalable} with the GPT-2 tokenizer~\citep{radford2019language}, following the same implementation as prior work~\citep{rutte2025generalized, sahoo2024simple, zhou2025next}. Due to the much smaller output layer, to maintain a fair comparison, we train two model variants: SMALL, with 17 layers, 12 attention heads, and hidden dimension 768, and BASE, with 27 layers, 16 attention heads, and hidden dimension 1024. The total number of parameters of SMALL and BASE are similar or smaller than the corresponding models of prior works. 

All models are trained using identical settings to~\citep{rutte2025generalized}: a context length of 512, batch size 512, and 500k optimization steps, corresponding to 131B training tokens. Training of SMALL and BASE is conducted on 8 NVIDIA RTX 6000 48GB GPUs using bf16 mixed-precision. Ablation studies exclusively uses SMALL;training is conducted on 4 NVIDIA RTX 3080 24GB GPUs using bf16 mixed-precision with a batch size of 128. 

Optimization uses Adam~\citep{adam} with $\beta=(0.9,0.99)$ and $\epsilon=10^{-9}$, an initial learning rate of $5\times10^{-4}$ with 10k-step linear warmup followed by cosine decay to 10\% of the initial rate. We apply weight decay of 0.02 and gradient clipping with norm 1.0. For training stability, loss weights $w_{t,m}$ and $w_{t,c}$ are clipped to 2.0 or 10.0 during training, but left unclipped when evaluating the ELBO.

All denominators in loss and ELBO weights are clipped to $10^{-4}$. Sequences longer than 512 tokens are randomly truncated, while shorter sequences are padded to length 512; padding tokens contribute to the training loss but are excluded from ELBO evaluation.

In addition, during vocabulary tree construction, we follow~\citep{zhou2025next} by adopting token embeddings from the pretrained GIDD Small/Base models~\citep{rutte2025generalized} for the iterative K-means clustering procedure.

\section{Joint Modeling of A Neighborhood of Tokens}
\label{appendix:joint-modeling}
Discrete diffusion language models often rely on a convenient independence assumption over tokens, namely
$p(x_{1:S})=\prod_{s=1}^S p(x_s)$, where $S$ is the sequence length. This assumption is largely a practical necessity for token prediction, since the target space otherwise grows exponentially with the vocabulary size. At the same time, prior work has shown that jointly modeling local neighborhoods of discrete states can be effective~\cite{chao2025beyond}; however, the ``states'' being modeled jointly are individual bits in an $n$-bit representation of a token rather than the token itself. As a result, such bit-level joint modeling does not break the token-independence assumption.

The proposal of TDLM, particularly its novel child prediction mechanism, makes joint token modeling feasible by reducing the effective prediction target space. As a result, an otherwise intractable joint prediction problem can be handled by restricting attention to a predefined token neighborhood. For example, jointly modeling $16$ consecutive tokens yields a target space of size $2^{16}=65{,}536$, which is reasonable relative to typical vocabulary sizes. In this section, we briefly discuss how TDLM can be used to achieve partial joint modeling of tokens.

\textbf{Prediction Model.} Specifically, we modify the prediction network $\mathbf{x}_\theta^h(z_t^h,t)$, which currently predicts a probability mass distribution over the set of children $\Gamma_{\downarrow}^{h}(z_t^h)$ of $z_t^h$, assuming that $z_t^h \in \Gamma_{\uparrow}^{h+1}(x)$. Now consider a partition of the sequence $z_{t,1:S}^h$ into $N-1$ non-overlapping neighborhoods $\{z_{t,S_i:S_{i+1}}^h\}_{i=1}^{N-1}$, specified by boundary indices $1=S_1<\cdots<S_N=S$. We then relax token-wise independence to neighborhood-wise independence, and model the joint distribution of the entire sequence of states at time $t$ as
\[
p(z_{t,1:S}^h) \;=\; \prod_{i=1}^{N-1} p\!\left(z_{t,S_i:S_{i+1}}^h\right).
\]

Consequently, we define the neighborhood prediction model
\[
\mathbf{x}_\theta^h\!\left(z_{t,S_i:S_{i+1}}^h,t\right):
(\mathcal{I}_h \cup \mathcal{I}_{h+1})^L \times [0,1]
\;\longrightarrow\;
\Delta\!\Big(\mathcal{C}\!\left(z_{t,S_i:S_{i+1}}^h\right)\Big),
\]
where
\[
\mathcal{C}\!\left(z_{t,S_i:S_{i+1}}^h\right)
\;=\;
\prod_{\ell=1}^{L}\Gamma_{\downarrow}^{h}\!\left(z_{t,\ell}^{h}\right)
\]
is the Cartesian product of the child sets for each token in the neighborhood, $\Delta(\cdot)$ denotes the probability simplex over the indicated set, and $L$ is the neighborhood length.

However, not all elements in $\mathcal{C}\!\left(z^h_{t,S_i:S_{i+1}}\right)$ are valid prediction targets given the input neighborhood $z^h_{t,S_i:S_{i+1}}$. There are two common cases. First, if a position has already been transmitted to the next threshold, then its value is fixed and nothing needs to be predicted at that position. Second, depending on the tree construction, some nodes may have fewer than $K$ children, so the feasible target space can be further reduced. Therefore, we mask out all impossible targets by assigning them zero probability, equivalently setting their logits to $-\infty$.

Concretely, for each position $\ell\in\{1,\dots,L\}$ in the neighborhood, define the feasible set of targets
\[
\mathcal{A}_\ell\!\left(z^h_{t,S_i:S_{i+1}}\right)
:=
\begin{cases}
\{z^h_{t,S_i+\ell-1}\}, 
& \text{if } z^h_{t,S_i+\ell-1}\in \mathcal{I}_{h},\\[4pt]
\Gamma_{\downarrow}^{h}\!\left(z^h_{t,S_i+\ell-1}\right),
& \text{if } z^h_{t,S_i+\ell-1}\in \mathcal{I}_{h+1}.
\end{cases}
\]
We define the valid joint target space as the Cartesian product of these per-position feasible sets,
\[
\mathcal{C}_{\mathrm{valid}}\!\left(z^h_{t,S_i:S_{i+1}}\right)
:=
\prod_{\ell=1}^{L}\mathcal{A}_\ell\!\left(z^h_{t,S_i:S_{i+1}}\right)
\;\subseteq\;
\mathcal{C}\!\left(z^h_{t,S_i:S_{i+1}}\right).
\]
The masked prediction distribution is then given by
\[
\mathbf{x}^{h}_{\theta}\!\left(z^h_{t,S_i:S_{i+1}},t\right)[c]
=
\begin{cases}
\displaystyle
\frac{\exp\!\big(E_{\theta}(c \mid z^h_{t,S_i:S_{i+1}},t)\big)}
{\sum\limits_{c' \in \mathcal{C}_{\mathrm{valid}}(z^h_{t,S_i:S_{i+1}})}
\exp\!\big(E_{\theta}(c' \mid z^h_{t,S_i:S_{i+1}},t)\big)},
& \text{if } c \in \mathcal{C}_{\mathrm{valid}}\!\left(z^h_{t,S_i:S_{i+1}}\right),\\[10pt]
0, & \text{if } c \notin \mathcal{C}_{\mathrm{valid}}\!\left(z^h_{t,S_i:S_{i+1}}\right).
\end{cases}
\]

 \begin{figure}[t]
  \centering
  \includegraphics[width=0.6\linewidth]{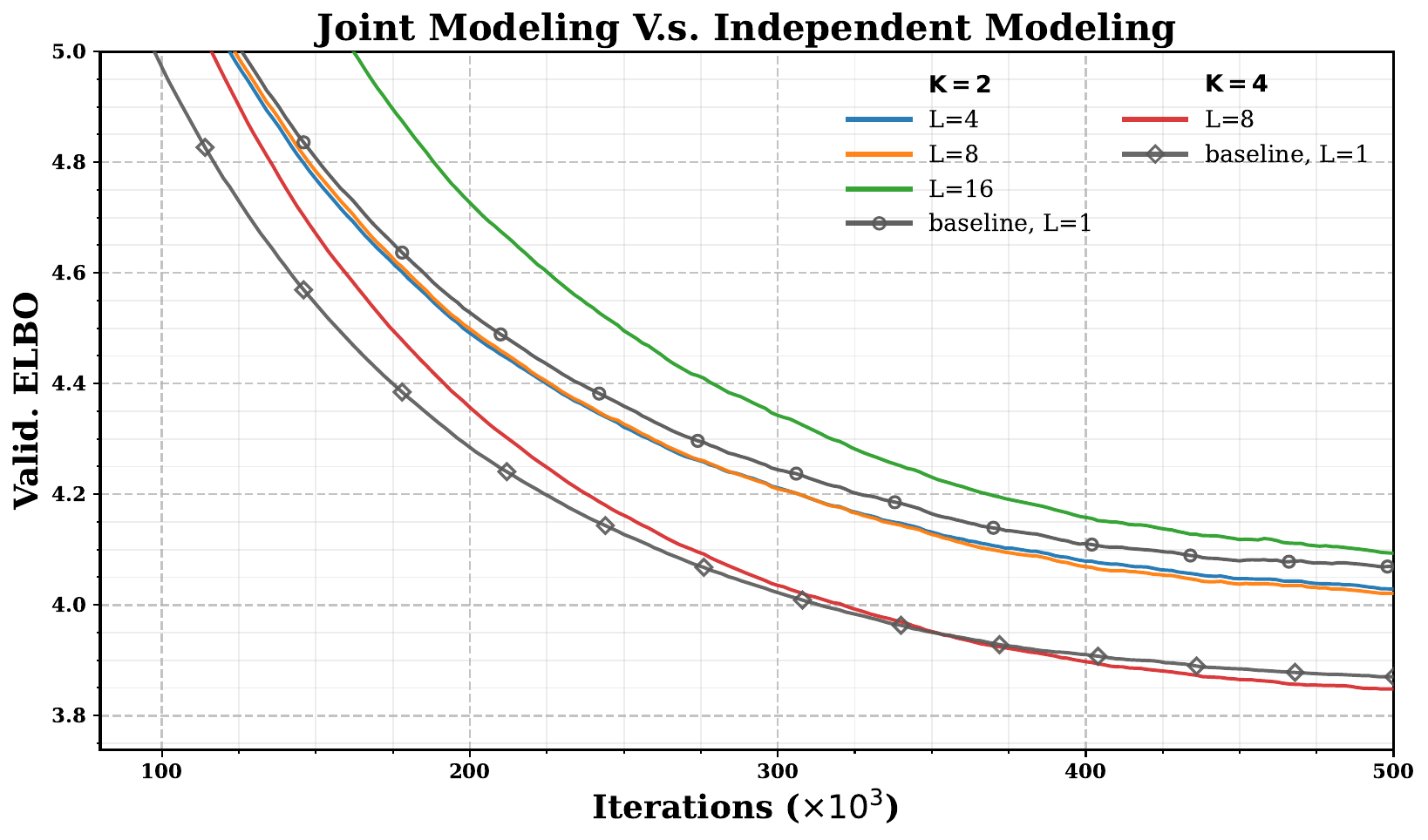}
  \caption{Smoothed validation ELBO for joint modeling with different branching factor $K$ and neighborhood length $L$.}
  \label{fig:joint-ablation}
\end{figure}
\textbf{Architecture Design.} The proposed joint modeling formulation requires only minimal modifications to the original backbone model. Although our definition of the predictor $\mathbf{x}^{h}_{\theta}\!\left(z^h_{t,S_i:S_{i+1}},t\right)$ enlarges its input domain to a neighborhood, this does not change the backbone model's input in practice, since the model already processes the entire sequence altogether. The key architectural change is therefore confined to the prediction head. To produce a joint distribution for a neighborhood, the model must first aggregate information from all tokens within that neighborhood into a single representation before the output layer. Two standard aggregation choices are average pooling and concatenation of the neighborhood token embeddings. In our experiments, concatenation consistently provides substantially better modeling capacity than average pooling. We attribute this to the stronger information bottleneck induced by averaging, which compresses the neighborhood into a single mean vector and can blur token-specific signals, whereas concatenation preserves per-token features and scales the representation dimension proportionally with the neighborhood length $L$.

\textbf{ELBO of Joint Modeling.} The ELBO for a neighborhood under joint modeling has exactly the same form as the ELBO for a single token in the original setting if we abstract each neighborhood and treat it as if it were a single token. However, the ELBO of the joint modeling is in nature $L$ times larger than that of independent-token baselines. Therefore, to enable a fair comparison with the baselines, we report a per-token ELBO for joint modeling by averaging the ELBO of joint modeling over all tokens in the neighborhood of length $L$.

\textbf{Preliminary Results.} We evaluate joint neighborhood modeling using the same SMALL DiT backbone as in our ablation studies. For $K=2$, we vary the neighborhood length $L\in\{4,8,16\}$; for $K=4$, we use $L=8$. The resulting joint classification space has size $2^{L}\in\{16,256,65{,}536\}$ for $K=2$, and $4^{8}=65{,}536$ for $K=4$. As shown in Fig.~\ref{fig:joint-ablation}, joint modeling exhibits slower early-stage convergence than token-independent modeling, but steadily closes the gap and can surpass the independent baseline in later training. Moreover, increasing $L$ tends to slow the convergence, which is expected since larger neighborhoods induce a substantially larger effective target space and require more optimization to learn local dependencies. Once learned, these dependencies translate into improved modeling capacity, leading to better final performance than independent token modeling.

However, the benefit from joint modeling is comparatively small relative to the effect of a varying tree structure: even the best $K=2$ joint-modeling configuration remains worse than the weakest $K=4$ setting. We conjecture that this gap may narrow at larger model scales, where deeper trees could improve faster and joint modeling may also benefits from additional capacity. We leave a systematic investigation of this scaling behavior to future work.

\section{Level-wise Weight Schedule}
\label{section:weight-schedule}
In ablation studies~\ref{sec:ablation}, we introduce level-wise weights to investigate the model's optimization in each level. Here we provide details on the linear and exponential weight schedules. 

We implement two simple level-wise reweighting schedules to upweight higher levels during training. Given the level indice $\beta \in \{0,\dots,H-1\}$, both routines compute a monotonically increasing weight $w(\beta)$, then mean-normalize the weights over all levels so that $\mathbb{E}[w(\beta)]= 1$, which stabilizes optimization by preserving the overall loss scale. The exponential schedule sets
\[
w(\beta)=\exp\!\bigl(\gamma_{\mathrm{exp}}\, \beta\bigr),
\]
where $\gamma_{\mathrm{exp}}$ controls how aggressively later levels are emphasized. The linear schedule instead uses
\[
w(\beta)=1+\gamma_{\mathrm{lin}}\cdot \frac{\beta}{H-1},
\]
providing a gentler and more interpretable increase. See Figure~\ref{fig:level-weight-schedules}.

\begin{figure}[thbp]
    \centering
    \begin{subfigure}[t]{0.45\linewidth}
        \centering
        \includegraphics[width=\linewidth]{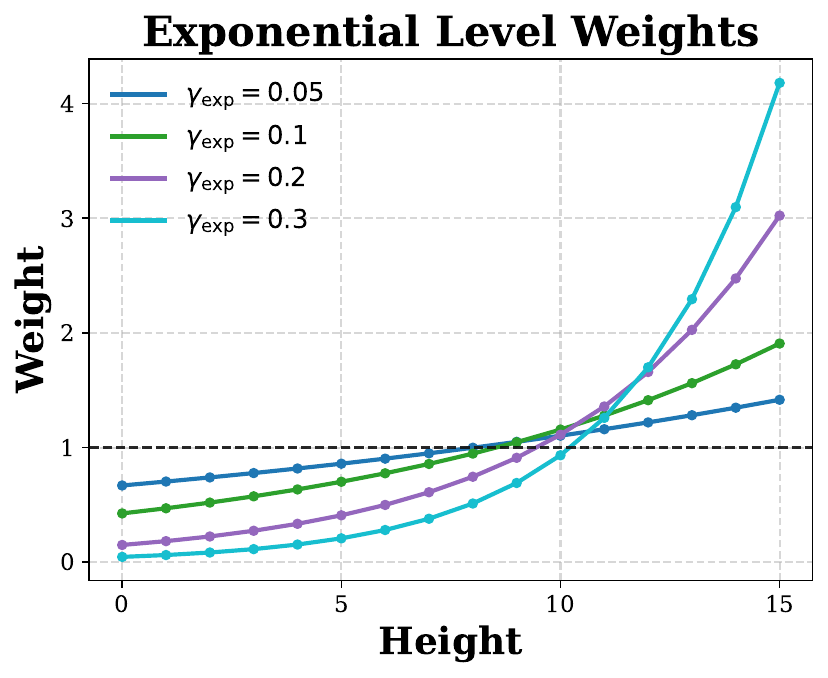}
        \caption{Exponential schedule.}
        \label{fig:level-weights-exp}
    \end{subfigure}
    \hfill
    \begin{subfigure}[t]{0.45\linewidth}
        \centering
        \includegraphics[width=\linewidth]{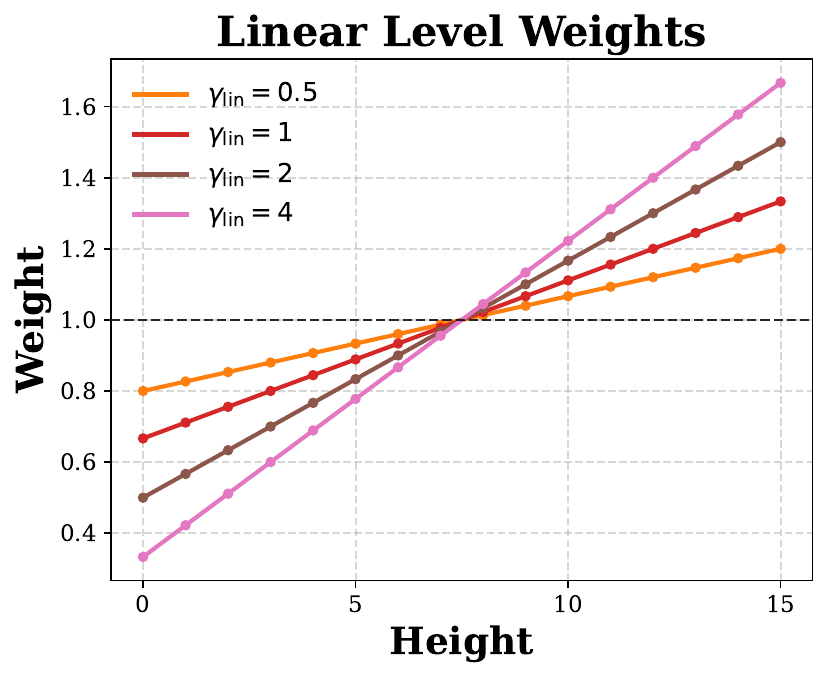}
        \caption{Linear schedule.}
        \label{fig:level-weights-lin}
    \end{subfigure}
    \caption{Level-wise weighting schedules used in our ablations.}
    \label{fig:level-weight-schedules}
\end{figure}

\section{Algorithms}
\subsection{TDLM Main Algorithm}
We follow the standard training framework for discrete diffusion language models and instantiate the corresponding loss under our tree structure, as summarized in Algorithm~\ref{alg:tdlm-heightwise-knary-loss}.

\begin{algorithm}[htbp]
\caption{TDLM Loss Function}
\label{alg:tdlm-heightwise-knary-loss}
\textbf{Input:} logits $\mathbf{L}\in\mathbb{R}^{B\times S\times K}$; tokens $\mathbf{x}\in\{0,\dots,V-1\}^{B\times S}$; noisy state $\mathbf{z}_t$; times $t$; tree structure $\mathcal{T}$; noise schedule $\mathcal{S}$ \\
\textbf{Output:} Time/height weighted loss map $\mathbf{J}\in\mathbb{R}^{B\times S}$ and ELBO map $\mathbf{E}\in\mathbb{R}^{B\times S}$ \\[2pt]

\textbf{Auxiliary functions:} \\
\hspace*{1.2em} $\textsc{NodeAtHeight}(\mathcal{T},x,h)$: the ancestor node of token $x$ at height $h$. \\
\hspace*{1.2em} $\textsc{ChildIndex}(\mathcal{T},x,h)\in\{1,\dots,K\}$: the ground truth child label for token $x$ at height $h$. \\
\hspace*{1.2em} $\textsc{ChildMask}(\mathcal{T},u)\in\{0,1\}^{K}$: which child slots of node $u$ exist. \\
\hspace*{1.2em} $\textsc{IsToken}(z)$: true if $z$ already corresponds to a vocabulary token. \\
\hspace*{1.2em} $\textsc{MaskToNegInf}(\mathbf{L},\Omega)$: replace $L_{b,s,k}$ with $-\infty$ whenever $\Omega_{b,s,k}=0$. \\[4pt]

\textbf{(1) Collect weights, height, and vocab info} \\
\hspace*{1.2em} $(\mathbf{w}_t,\tilde{\mathbf{w}}_t)\leftarrow \textsc{TimeWeights}_{\mathcal{S}}(t)$ \\
\hspace*{2.2em} \emph{// $\mathbf{w}_t$: raw time weight; $\tilde{\mathbf{w}}_t$: clipped/stabilized time weight for training} \\
\hspace*{1.2em} $\mathbf{h}\leftarrow \textsc{HeightFromT}_{\mathcal{S}}(t)$
\\\hspace*{1.2em} $\mathbf{w}_h\leftarrow \textsc{HeightWeights}(\mathbf{h})$ \\
\hspace*{1.2em} $\mathbf{vocab}\leftarrow \textsc{IsToken}(\mathbf{z}_t)$ \\[6pt]

\textbf{(2) Construct the classification problem at height $\mathbf{h}$} \\
\ForEach{position $(b,s)$}{
    $u_{b,s}\leftarrow \textsc{NodeAtHeight}(\mathcal{T},x_{b,s},h_b)$ \quad \emph{// current node of token $x_{b,s}$ at height $h_b$} \\
    $y_{b,s}\leftarrow \textsc{ChildIndex}(\mathcal{T},x_{b,s},h_b)$ \quad \emph{// groundtruth label of children prediction for $x_{b,s}$ at height $h_b$} \\
    $\Omega_{b,s,:}\leftarrow \textsc{ChildMask}(\mathcal{T},u_{b,s})$ \quad \emph{// valid children for this node} \\
    $v_{b,s}\leftarrow \mathbb{1}\!\left[z_{t,b,s}= u_{b,s}\right]$ \quad \emph{// not yet transitioned at height $h_b$} \\
    $\mathbf{valid}_{b,s}\leftarrow (\neg \mathbf{vocab}_{b,s})\ \wedge\ v_{b,s}$ \quad \emph{// valid states are not vocabulary and not transmitted within the level}
}
\textbf{(3) Masked cross-entropy} \\
\hspace*{1.2em} $\mathbf{L}'\leftarrow \textsc{MaskToNegInf}(\mathbf{L},\Omega)$ 
\quad \emph{// invalid children ($\Omega=0$) receive zero probability under softmax} \\
\hspace*{1.2em} $\boldsymbol{\ell}\leftarrow \mathrm{CE}(\mathbf{L}',\mathbf{y})$ \quad \emph{(per-position, no reduction)} \\
\hspace*{1.2em} $\boldsymbol{\ell}\leftarrow \boldsymbol{\ell}\odot \mathbf{valid}$ \\[4pt]

\textbf{(4) Apply weights} \\
\hspace*{1.2em} $\mathbf{J}\leftarrow \boldsymbol{\ell}\odot \tilde{\mathbf{w}}_t\odot \mathbf{w}_h$ \\
\hspace*{1.2em} $\mathbf{E}\leftarrow \boldsymbol{\ell}\odot \mathbf{w}_t$ \\[2pt]

\textbf{return} $\mathbf{J}$ and $\mathbf{E}$
\end{algorithm}

\subsection{Tree Construction Algorithm}
We adopt the clustering algorithm provided in HDLM \citep{zhou2025next}. Using this clustering algorithm (including clustering and size control), we formulate Algorithm \ref{alg:tree_construct} to build the semantic tree used to train TDLM. 
\newpage
\begin{algorithm}[htbp]

\caption{Fixed-$K$ Semantic Vocabulary Tree Construction}
\KwIn{Token embeddings $\{\mathbf{e}_i\}_{i=1}^V$, branching factor $K$, maximum depth $D_{\max}$, cluster size ratio $\min/\max$}
\KwOut{A $K$-ary semantic tree $\mathcal{T}$ and token-to-leaf paths}

\textbf{Initialization:}\;

Initialize root node containing all $V$ tokens\;

Compute root centroid as the mean embedding\;

Initialize a queue with the root node\;

\While{queue is not empty}{
    Pop a node $n$ with token set $\mathcal{M}_n$\;
    
    \If{stopping criteria satisfied (leaf size or depth limit)}{
        Continue\;
    }
    
    \If{$|\mathcal{M}_n| < K$}{
        Partition $\mathcal{M}_n$ into $|\mathcal{M}_n|$ singleton children\;
    }
    \Else{
        Apply semantic $K$-means clustering to embeddings of $\mathcal{M}_n$\;
        (Algorithm 1 and 2 of \citet{zhou2025next})\;
        
        Assign tokens to clusters\; 
    }
    
    \ForEach{child cluster}{
        Create a child node with its token subset and centroid\;
        
        Add the child to the queue\;
    }
}

Compute token-to-leaf paths by traversing the tree\;

\If{forcing a complete tree}{
    Determine the maximum leaf height\;
    
    Extend all shorter leaves by repeating the final branching choice until all leaves reach the same height\;
    
    Recompute token-to-leaf paths\;
}

\Return{semantic tree $\mathcal{T}$ and token paths}

\label{alg:tree_construct}
\end{algorithm}

\begin{table}[H]
\centering
\footnotesize
\setlength{\tabcolsep}{6pt}
\renewcommand{\arraystretch}{1.2}
\begin{tabularx}{\linewidth}{X}
\toprule
\textbf{Qualitative Examples: Generated Text With 512 Sampling Steps} \\
\midrule
Main difference is that much like the economic ladder movement sparked in Ferguson, that small town, citizens concerned about police misconduct, from labor activists to hardcore progressive institutions, all gathered. The message of income linkage was all at the center, the widespread Streeck discontent, communities that felt much better off than individuals once were. Their voice was more relevant than activists giving rise to the message, but the main driver was inherent inequality, and that is a large part of what was the gap between poverty and income.
You likely didn't see a bold income linkage message from the many pro-dissinity Democrats who urge economic polarization and there just weren't many answers calling for policy that's harmful for the needy, either.

A message: Income linkage led by affirmative Democratic leaders might not resonate with low-income individuals who are housing survivors in their lives and have deeper ties to the Left where their previous-class status still allows. Further to these aspects, I argued that if the income scale were off-course for theoretical days it would have been left to the working and middle class– (these elements that even the with the most political willingness and ability feel marginalized by left-progressive leadership long deserve at the bottom of the income scale). The message might have expanded and supersized by taking some finer steps that made sense in fighting intolerance and poverty. The message would have seen a left step up if centrists opted to blame the plight of some low-income people on others, but even then the progressive movement likely wouldn't have topped paychecks among all the lowest-income earners.

We would have made policies anti-poverty if we appealed to only the least, those groups with more at stake in engaging in poverty reduction discourse. We needed to reject these political science to recognize that certain anecdotes were more uncomfortable in the middle of (status.) Burrow specifically foresaw that the Obama administration's opportunities to act were controlled by the anti-poverty speculation we are now consistently dealing with. (Especially Washington policymakers who this moment have been willing to overlook a cohesive progressive agenda.)

We dispelled the decision to make legal requirements for the poor more economically viable, welfare ambassador time was spent personally with meaning more cash for community organizations and the race to a grass of good is surely over. Burrow is right: we were policies that hold the stronger voice for the affirmative in the larger segment that fails to involve message reaching the poor or status. Income disparities discourse were not extraneous to ending \\
\midrule
Gender advocacy on workers relationship prime in life should be: The daunting struggle of life. Women with single mothers, women who are single professionals, and, in particular, busy men most easily ought to prefer leading a wonderful life to, as we have written previously, women imposed by more extraordinary circumstances and accomplures. Yet, both men run more well than women, while gender attitudes on differences of adolescence break up entirely between the two and into one getting more firmly briefed, the other is often emotionally closer to men. Both seek a lot farther from the ideal of life for all, while both demographic groups still survives.

As for Scientific opinion about adolescence—in short is ``Live Long, Die Long." In commonly the wisdom and insight of teen Vikings readers that Dr. Dana Reeer had compared to the scientific evidence, like childhood drug toxicity, or thyroid problems. or, even more expound by experts, research that finds that children and teens, even people in their late 30s — and not adults — may have been most prone to physical, emotional, violence, and hostile home environment.

The signs of adolescent health and well-being set off a mutual struggle between those fears and attitudes of teens that belongs to both sexes, Reeer writes: For those who have research find to adults, after 30 or young adulthood, think of body mass index and life expectancy as such fifth percent degree, while teens themselves are in the 70s. The sciences, attitudes and mentations of adults, critical focus into childhood undo the accumulated patterns of girls and women now and over their lives and and place an increased value on women's emotional identity and well-being.

Regardless of expenditures and people's incomes, women are more hands-on on their lives than men, and more women are able to resolve the baseline blurred in emotions nor behaviors, Dr. Reeer writes for the Disbalance, as well as to address the differences between men and women. But to add to that indignant attitude uttered by both genders toward people who are later jealous, more girls when single are able to avoid focus as prepared for these first impulses of adulthood, and more able to weave an emotion removed longer from ignoring the differences of youth and partner than dispatch on keeping sides of the root cause of neglect that momentarily affects so much of your baseline survivability.
But you cannot only focus on emotions and critical emotion and filling comfort gaps as we have written in Getting started on Enthusiasm Approaches, Mistakes, Trans \\
\bottomrule
\end{tabularx}
\caption{Qualitative examples generated at 512 sampling steps.}
\label{tab:sampling-examples}
\end{table}

\begin{table}[H]
\centering
\footnotesize
\setlength{\tabcolsep}{6pt}
\renewcommand{\arraystretch}{1.2}
\begin{tabularx}{\linewidth}{X}
\toprule
\textbf{Qualitative Examples: Generated Text With 1024 Sampling Steps} \\
\midrule

Greenfield, for his part, cited internet-based super PAC reports of second-hand contributions, judged by pundits, up to \$22 million in the last election from super PACs. But he cited the Independent Spending Council considered an election spending association, in recognition that super PACs stand foremost as top-spending" groups on political activities, that registered firms should accept their spending and conducting organizations.

Super PACs now have the Federal Election Commission and redistricting Democrats with the recent Supreme Court decision as their reins, agree.

Republican leaders like the LP cheered the Citizens United ruling and observed that Obama's federal campaign finance laws are comfortably toward being struck down, again.

``I think of how the Republican presidential candidates see a positive step here in Citizens United," said Republican National Committee Chairman David Walker Jr. R-Ohio, of clogging Obama's individual mandate ``wizards" toward the ``enormancy and ignorance" of the Capitol Hill office of Senate Majority Leader Harry Reid Jr., D-California. But even he's not even going to seek to gut the opposing precedent of Citizens United.

In the protest of Democrat leaders stepping to the party after Long's announcement included Capitol Hill's Jack Reed (R., Virginia). A chief, conservative font included: Rep. Nancy Pelosi, Chairwoman of the U.S. Securities \& Exchange Commission.

Stern said there has been financial consensus in the House of Representatives to get him to recuse himself from benefiting from secret money used by super PAC groups for political activity and ban groups and members from voting against them.

The coalition created to protest against Tuesday's decision was organized between Long and Pai, stars and various top actors, including the Allison Wives magazine run, who via the LP.

They endorsed the Mega Finance statement as ``an inclusive, politically empowered multicultural coalition mixed with an intelligent record of interventions on the racial and sexual diversity of the political financing landscape." The firm includes Mega Finance Chairman Mark Michael Ficloma; Richard Mache Perez, the chief economist of Prudential Advisors; Taylor Jin, VP Senior Chairman at Federal Bank of Chicago and Hallie Douglas (MO), Chairman of National Rentali Corp.; and Jacqueline Rosen, president and financial advisor\\ \midrule
The health care system is also very problematic. Private insurance has received an ample chunk too little for it to open drug facilities in some parts of the market, and Private rated providers cannot at first take return, on subscriber money to open free rehab facilities. Fortunately, there is now part of the Affordable Care Act's tax breaks for certificates in Pharmacology Management Education (POAC) and other health care necessities. Obamacare's subsidies will directly support researchers at all federal levels for reduction in traffic deaths due to opioids. The subsidies will help testing of "planning techniques" and drug-education programs for inmates seeking outpatient treatment.

One possibility for the largest benefit of ministerial prisons specifically when considering drug punishments is the ability to assist inmates, who want to be helped, into treatment via co-program of forbearance to take some psychapeptic drugs. Programming based on a lesser degree of forbearance is also typical for most medical insurance companies to purchase some inexpensive prescription benefit.

Better still, however, is changes in medical insurance laws. A much newer short-term insurance plan can render treatment effective in varies and drug abusers by removing expensive treatment for those patients. Often, additional long-term treatment is preceded by a medical spending of outreach premiums 36\% or much more reduced, which lead to changes in daily legal killers that propelled opioid users and changed drug symptoms. Sometimes, it is appropriate to limit ministerial-lock policies by other means, such as if TTE+ or tertiary palliative, while purpose that do not necessarily require its expected derive.

Based on federal offender registries stat is about 1.2 million Americans die from overdose each year.

The interpretation system of drug abuse from those that have really no delusion is in addition to multiple consistent requirements: ward off the black-and-white idea of sufficient drug funding. However, the recent BBC documentary study showed that 80\% of prison opioids have far less major issues than even a first year's ``white" virus in the commonly-ignored norm. It found that the rate of past reports of adverse effects seen for every 5,000 IHA members is 12\%, and the 2015 estimated rate of urine inhaled hazardous is 3.9\%.

Those crucial findings which reveal Bad-ass Prison Revenge will likely lead to a very restrictive medical norm. The public suffers from a future in which 4 drug prison gang and 5 giro-club spread for most drug dealing and where new convictions, convicts not released, are quickly giving drug
 \\
\bottomrule
\end{tabularx}
\caption{Qualitative examples generated at 1024 sampling steps.}
\label{tab:sampling-examples}
\end{table}

\end{document}